\documentclass{article}
\usepackage{iclr2027_conference,times}
\usepackage{amssymb,amscd,amsthm, verbatim,amsmath,color,fancyhdr, mathrsfs}
\PassOptionsToPackage{numbers}{natbib}
\usepackage{graphicx}
\usepackage{subcaption}
\usepackage{turnstile}
\usepackage{enumitem}
\usepackage{wrapfig}
\usepackage{empheq}
\usepackage[colorlinks=true, linkcolor=blue, citecolor=blue, urlcolor=blue, pagebackref]{hyperref}

\allowdisplaybreaks
\usepackage{natbib}
\usepackage{booktabs}
\usepackage{multirow}
\usepackage{pifont}
\usepackage{bbm}
\usepackage[table]{xcolor}
\usepackage{tikz}
\usetikzlibrary{arrows.meta}
\newcommand{\cmark}{\ding{51}}%
\newcommand{\xmark}{\ding{55}}%
\newtheorem{thm}{Theorem}
\newtheorem{prop}{Proposition}
\newtheorem{lem}{Lemma}

\newtheorem{defn}{Definition}

\newcommand{\singleton}[1]{\{#1\}}
\newtheorem{appthm}{Theorem}
\newtheorem{appprop}{Proposition}
\newtheorem{applem}{Lemma}

\usepackage{amsmath,amsfonts,bm}

\def\1{\bm{1}}

\DeclareMathAlphabet{\mathsfit}{\encodingdefault}{\sfdefault}{m}{sl}
\SetMathAlphabet{\mathsfit}{bold}{\encodingdefault}{\sfdefault}{bx}{n}

\DeclareMathOperator*{\argmin}{arg\,min}

\usepackage{placeins}
\usepackage{url}
\title{Rethinking Representation Geometry in Supervised Contrastive Learning under Long-Tailed Data}
\author{
 Joonhyeok Jeong \qquad
 Jy-yong Sohn \qquad
 Kibok Lee \\
 Department of Statistics and Data Science, Yonsei University \\
 \texttt{\{alexjeong10, jysohn1108, kibok\}@yonsei.ac.kr}
}
\iclrfinalcopy
\begin{document}
\maketitle
\begin{abstract}
Supervised contrastive learning and its balanced variants have demonstrated their effectiveness for representation learning under long-tailed distributions. However, although prior analyses suggest that their optimal representations may exhibit a regular simplex geometry, 
whether such a geometry is indeed exhibited
remains insufficiently understood. 
In this work, we present a theoretical framework for these objectives that characterizes whether a regular simplex is formed at the optimum. By applying this to representative objectives, we show that these objectives fail to form a regular simplex at the optimum because they remain dependent on the number of samples per class. Motivated by this analysis, we introduce structural balancing that can be applied to contrastive objectives, and propose a fully balanced contrastive objective that provably forms a regular simplex geometry at its optimum. Both theoretical and empirical results demonstrate that the proposed method yields a more balanced representation space and improves performance on long-tailed recognition and transfer learning tasks.
\end{abstract}

\section{Introduction}
Deep learning has achieved remarkable success over the past decade, yet much of this progress relies on highly curated, balanced datasets \citep{krizhevsky2009learning,deng2009imagenet}, while real-world data often exhibit long-tailed distributions, in which a small number of head classes are represented by abundant samples while a large number of tail classes contain only a few samples \citep{van2018inaturalist}.
Models trained under such long-tailed distributions tend to learn skewed decision boundaries, which in turn degrades generalization performance on tail classes \citep{cao2019learning, ren2020balanced}.

Supervised contrastive objectives have shown promising results in learning representations from long-tailed distributions \citep{kang2020exploring, wang2021contrastive, cui2021parametric, li2022targeted, zhu2022balanced, cui2023generalized, hou2023subclass, du2024probabilistic, de2025long, yi2025geometry}. As in balanced distributions \citep{graf2021dissecting, yiconvergent}, characterizing an optimum of a loss is important for understanding the representation geometry the loss favors. However, the representation spaces induced at the minimum of these objectives remain insufficiently understood under long-tailed distributions. Among existing methods, balanced contrastive learning (BCL) \citep{zhu2022balanced} showed that the loss is minimized if and only if the class centers form a regular simplex, suggesting that the ideal geometry remains optimal even under long-tailed distributions.

In this paper, we establish a theoretical framework for determining whether a regular simplex is a global minimizer of a supervised contrastive objective under long-tailed distributions. We first reduce the analysis from the feature space to the class-wise center space by deriving a lower bound of the loss, whose equality condition corresponds to class collapse. We then show that a center configuration that does not minimize the lower bound cannot minimize the original loss. Finally, we derive a condition that the lower bound must satisfy 
for a regular simplex to be its minimizer. We apply this framework to two widely used objectives for long-tailed recognition, SCL \citep{khosla2020supervised} and BCL, and prove that the regular simplex is not a global minimizer of either loss. Fig.~\ref{fig:regular-simplex} shows this behavior with ImageNet samples.
\begin{figure}[t]
    \centering
    \begin{subfigure}{0.48\textwidth}
        \centering
        \begin{minipage}{0.32\linewidth}
            \centering
            \includegraphics[width=\linewidth]{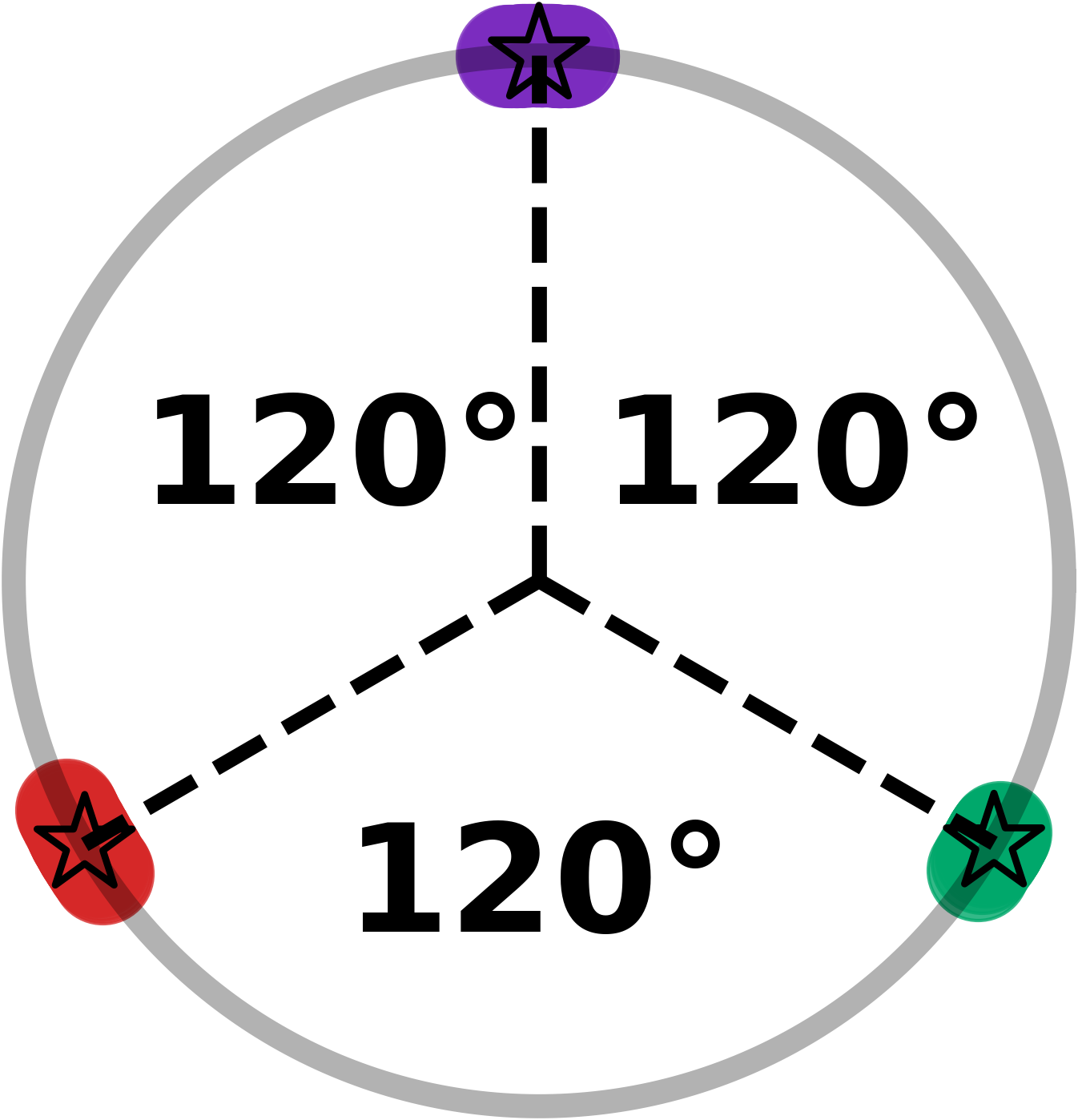}\\[0.5ex]
            \footnotesize SCL
        \end{minipage}\hfill
        \begin{minipage}{0.32\linewidth}
            \centering
            \includegraphics[width=\linewidth]{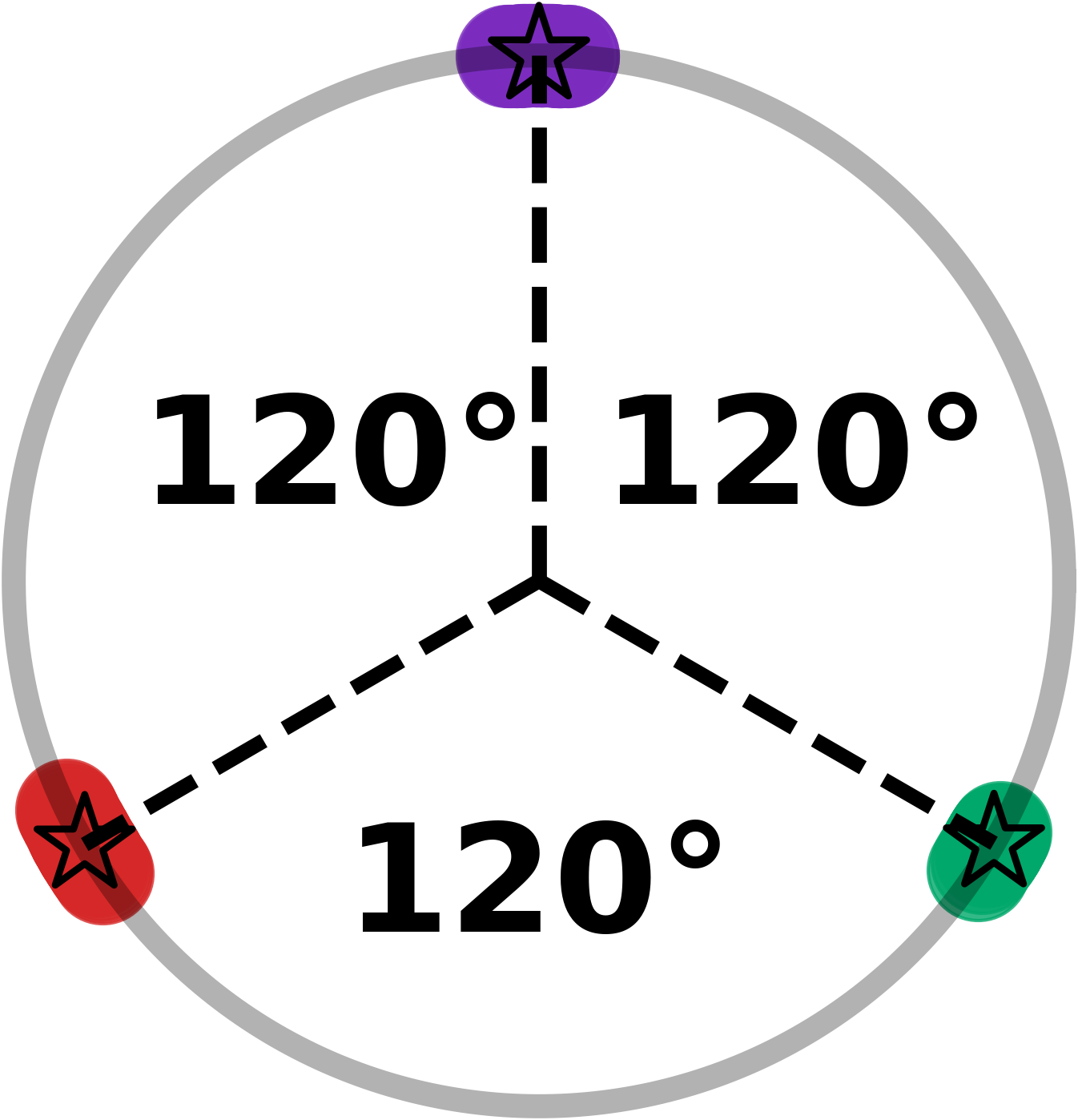}\\[0.5ex]
            \footnotesize BCL
        \end{minipage}\hfill
        \begin{minipage}{0.32\linewidth}
            \centering
            \includegraphics[width=\linewidth]{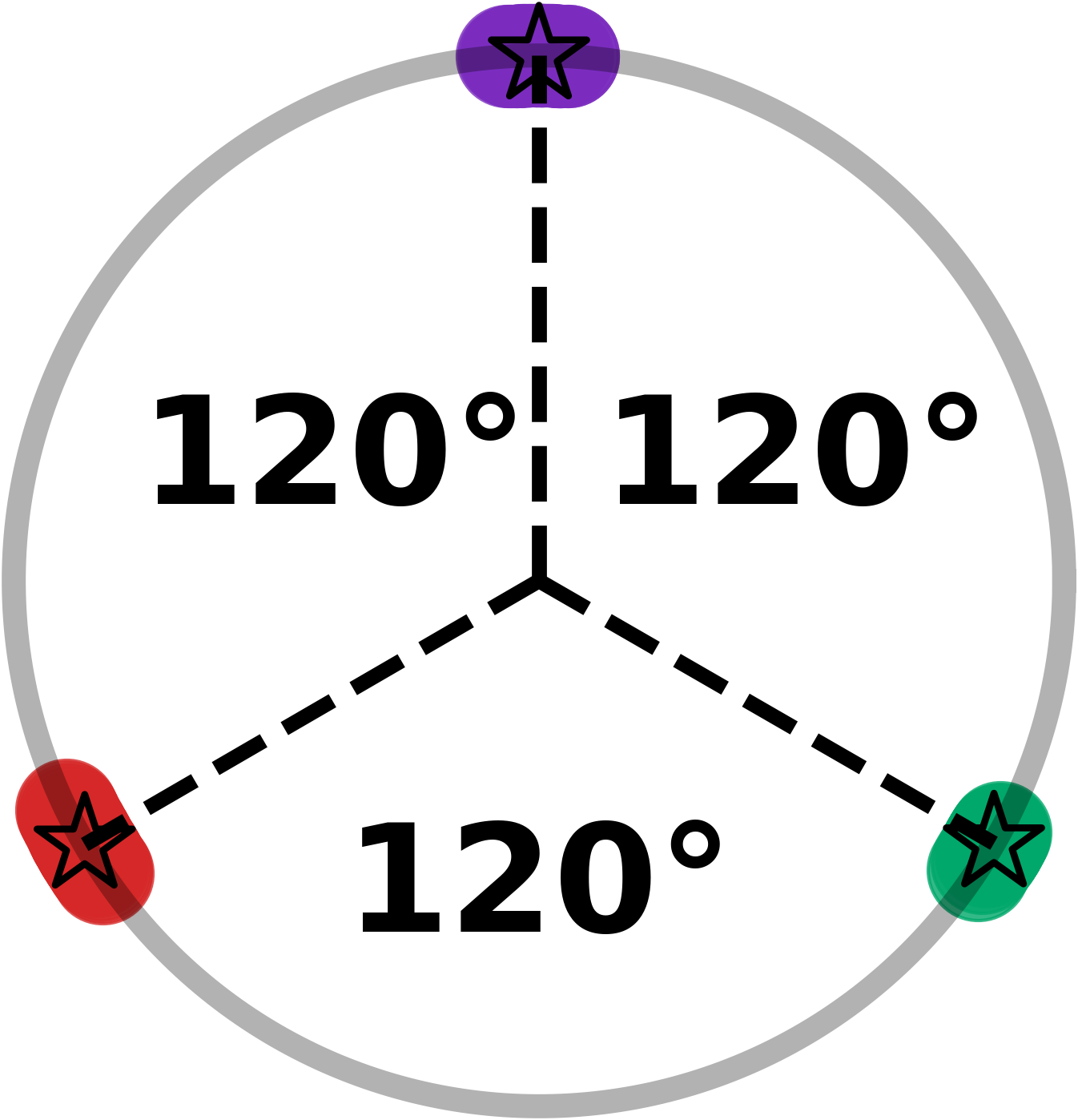}\\[0.5ex]
            \footnotesize FBCL (Ours)
        \end{minipage}
        \hfill\vrule \hfill
        \caption{\textcolor[HTML]{7530AC}{1}:\textcolor[HTML]{B3352F}{1}:\textcolor[HTML]{4BA570}{1}}
    \end{subfigure}
    \begin{subfigure}{0.48\textwidth}
        \centering
        \begin{minipage}{0.32\linewidth}
            \centering
            \includegraphics[width=\linewidth]{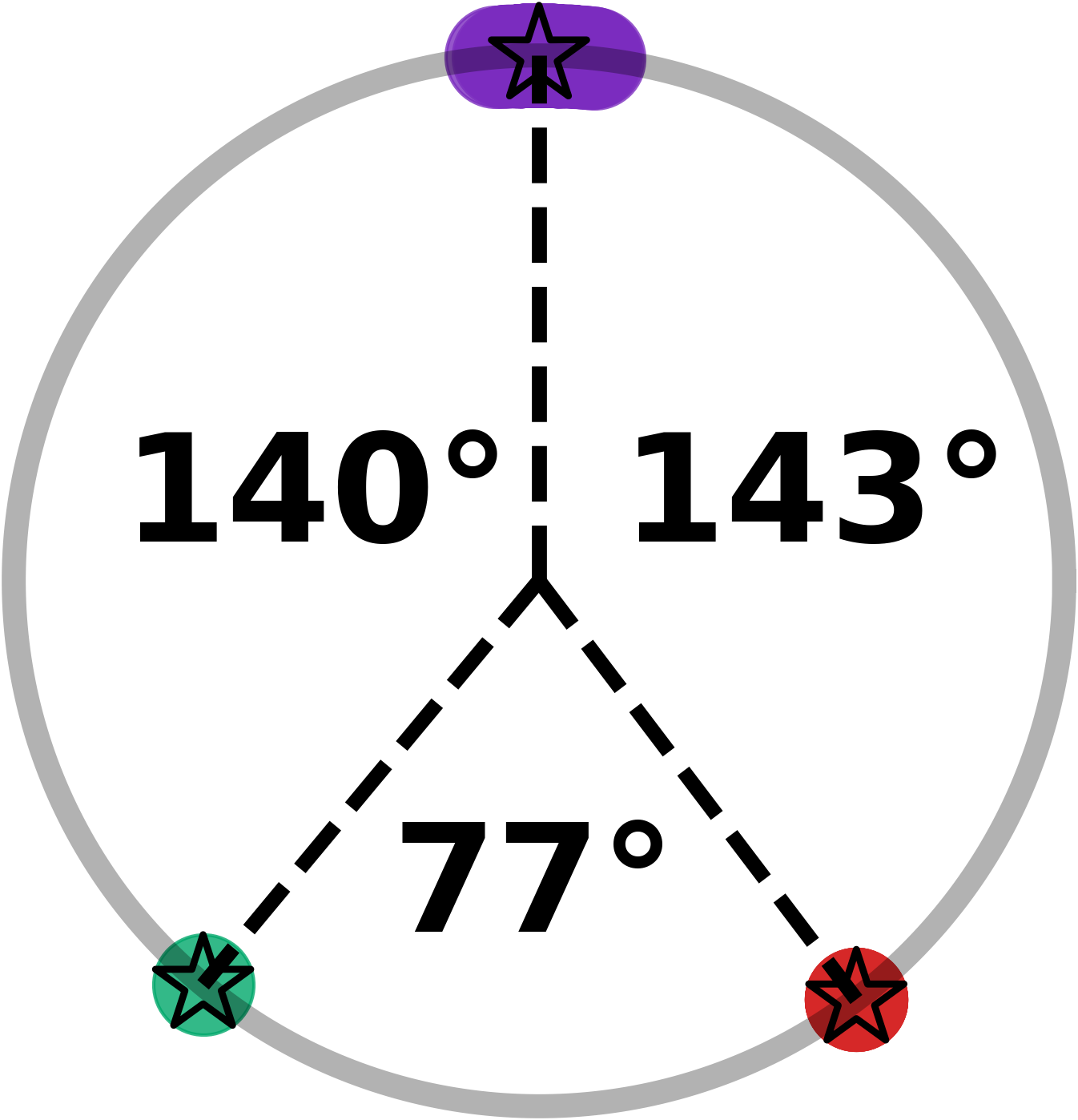}\\[0.5ex]
            \footnotesize SCL
        \end{minipage}\hfill
        \begin{minipage}{0.32\linewidth}
            \centering
            \includegraphics[width=\linewidth]{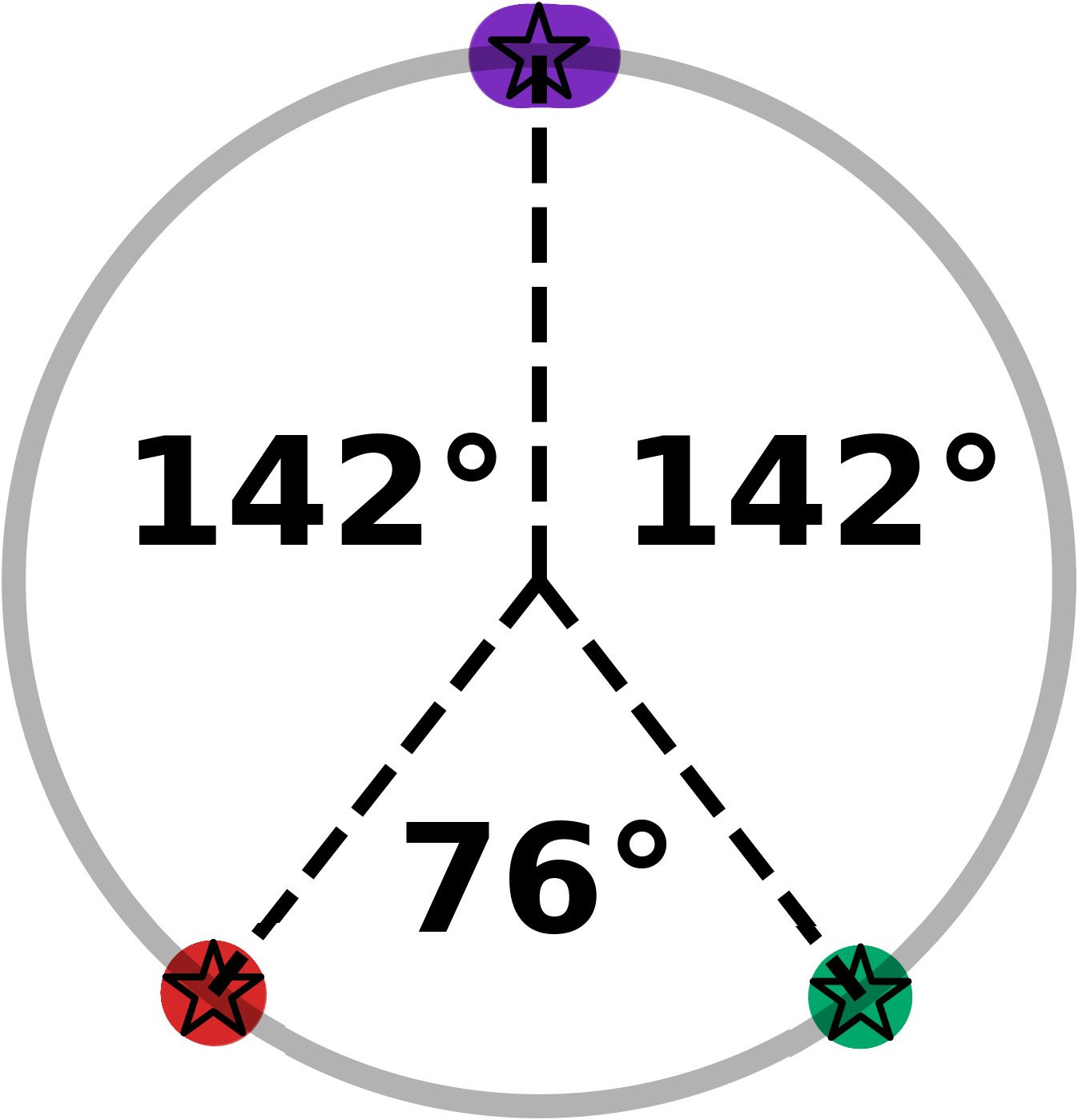}\\[0.5ex]
            \footnotesize BCL
        \end{minipage}\hfill
        \begin{minipage}{0.32\linewidth}
            \centering
            \includegraphics[width=\linewidth]{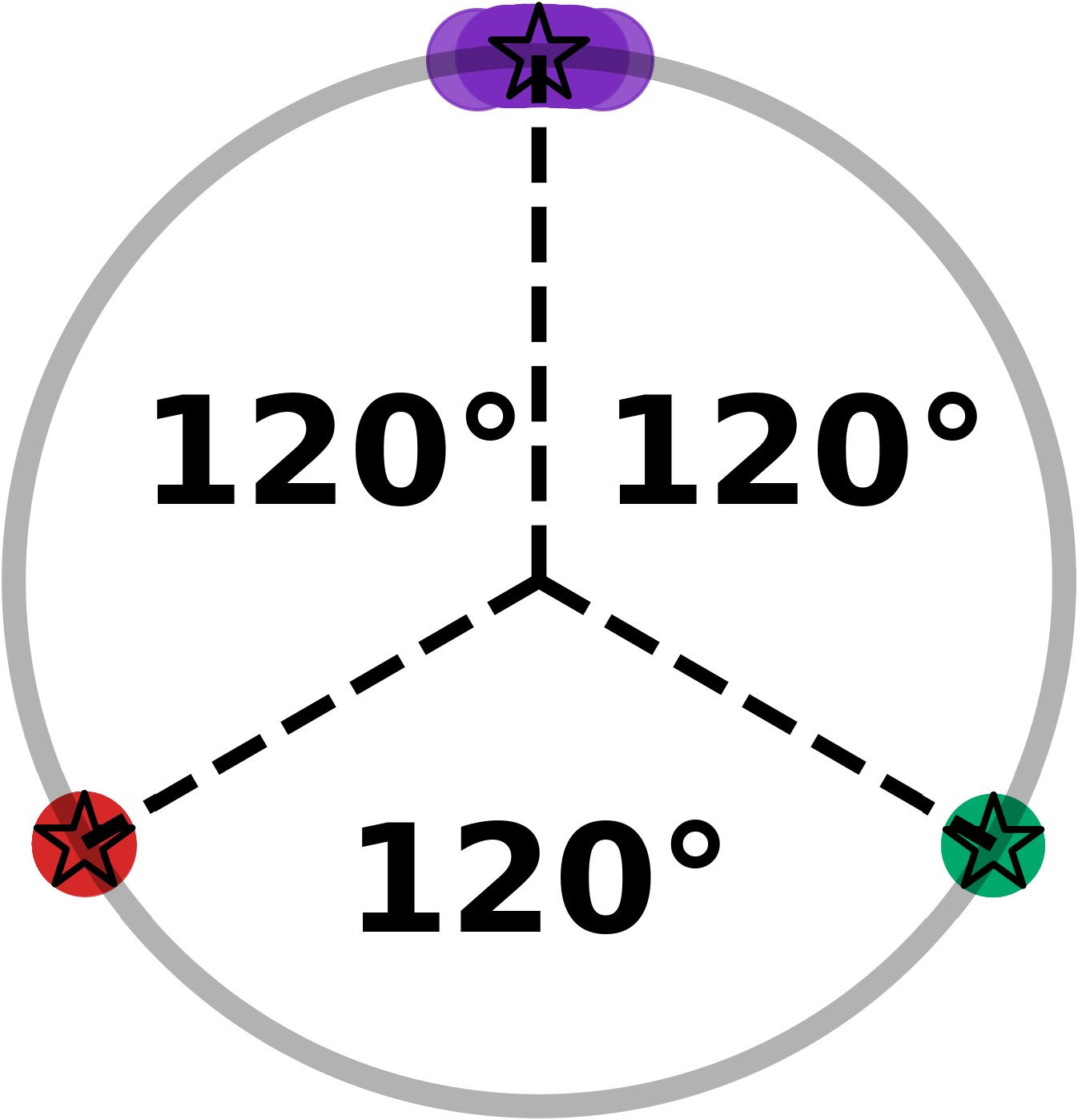}\\[0.5ex]
            \footnotesize FBCL (Ours)
        \end{minipage}
        \caption{\textcolor[HTML]{7530AC}{1000}:\textcolor[HTML]{B3352F}{10}:\textcolor[HTML]{4BA570}{1}}
    \end{subfigure}
    \caption{Visualizations of embedding spaces learned by three contrastive objectives (SCL, BCL, and FBCL) on a subset of ImageNet: three random classes sampled under (a) a balanced setting and (b) an imbalanced setting with class cardinality ratios of 1:1:1 and 1000:10:1, respectively. 
    Angles are measured between class centers. Under (a), all three methods form a regular simplex; under (b), only FBCL forms it, whereas SCL and BCL fail.}
    \label{fig:regular-simplex}
\end{figure}

We further investigate why this result holds by examining how class imbalance affects the loss. From a balancing perspective, SCL and BCL adjust the loss only within the contrastive loss defined for each anchor, the sample for which the contrastive loss is computed. Consequently, as the losses for individual anchors are aggregated into the overall loss, the absence of balancing this aggregation leaves the loss dependent on the number of samples per class. This dependence is inherited by the lower bound, preventing the regular simplex from being a minimizer of the objective. Motivated by this observation, we present structural balancing, a method applicable to supervised contrastive objectives that balances the aggregation across anchors and removes the remaining dependence on the number of samples per class. Based on structural balancing, we propose Fully Balanced Contrastive Learning (FBCL), whose global minimizer is a regular simplex.

We also analyze the effect of structural balancing on the representation space by examining how the loss penalizes the intra-class variance of each class as the loss is minimized. Under SCL and BCL, the coefficient of the intra-class variance penalty is approximately proportional to class frequency, whereas under FBCL it is approximately independent of it. This suggests more balanced contraction across classes, which is further supported by our empirical analysis of intra-class variance. We further evaluate FBCL with other contrastive objectives and apply structural balancing to other long-tailed recognition methods. Results show that FBCL and structural balancing can be effectively combined with other methods and improve their performance.

Our main contributions are summarized as follows:
\begin{itemize}[leftmargin=*, itemsep=0pt, topsep=0pt, partopsep=0pt]
    \item We establish a framework for determining whether a regular simplex is a global minimizer of a supervised contrastive objective under long-tailed distributions. Applying this framework to SCL and BCL, we prove that the regular simplex is a global minimizer of neither loss.
    \item We observe that this failure occurs as a result of the remaining dependence on the number of samples per class in the loss. Based on this observation, we present structural balancing and propose Fully Balanced Contrastive Learning (FBCL), whose loss removes this dependence and admits the regular simplex as its global minimizer.
    \item We empirically evaluate FBCL on long-tailed recognition and transfer learning, and further examine structural balancing in combination with other methods. The results show consistent improvements in classification accuracy and transferability across these settings.
\end{itemize}
\section{Related Work}
\label{gen_inst}
\paragraph{Long-Tailed Recognition.}
Numerous approaches have been proposed to address long-tailed recognition. Re-sampling \citep{he2009learning,buda2018systematic,shi2023re} and re-weighting classification losses \citep{huang2016learning, ren2018learning,cui2019class} were earlier methods in which the former balances the number of training samples and the latter assigns different weights to each training sample.
\citet{zhou2020bbn} and \citet{kang2019decoupling} propose a two-branch network and a two-stage learning framework, respectively, and show that such designs often yield better representations. 
Several works further demonstrate the effectiveness of these frameworks \citep{wang2021contrastive, zhong2021improving, li2022long, li2024feature}. Logit adjustment methods \citep{cao2019learning,menon2021longtail,kini2021label, li2022long,wang2023unified} enlarge the margins for tail classes. Transfer learning methods \citep{yin2019feature, li2024feature} enrich tail class features by leveraging abundant information from head classes. Recently, fine-tuning foundation models \citep{tian2022vl, DBLP:conf/iclr/DongZYZ23, DBLP:conf/icml/Shi00SH024} has proven effective with efficient training schemes, yet the impact of imbalanced data on fine-tuning is not fully understood.

\paragraph{Contrastive Learning under Long-Tailed Distributions.}
Contrastive learning \citep{he2020momentum, chen2020simple} has shown significant progress in representation learning, mostly under balanced data distributions. To address long-tailed distributions, 
Hybrid-SC \citep{wang2021contrastive} incorporates the SCL loss into the feature learning branch of a two-branch network, leveraging the quality of its learned embeddings to obtain more discriminative representations. 
However, subsequent studies have shown that SCL itself is biased toward head classes and yields a skewed representation space \citep{kang2020exploring,cui2021parametric,zhu2022balanced, li2022targeted}.
To mitigate this issue,
PaCo \citep{cui2021parametric} introduces parametric class-wise learnable centers, KCL \citep{kang2020exploring} uses the same number of positive samples for all classes, and BCL \citep{zhu2022balanced} balances the negative terms so that each class contributes equally to the repulsive force. TSC \citep{li2022targeted} assigns uniformly distributed class targets to attract the features of each class. 
\citet{yi2025geometry} 
demonstrates that the geometry of the representation obtained by minimizing the SCL loss is skewed under a specific scenario where all classes share the same number of samples except for one majority class.
Different from previous works, we provide a framework for analyzing whether a regular simplex is a global minimizer of supervised contrastive learning objectives under long-tailed distributions, based on the form of their lower bounds. Using our framework, we study two representative objectives: the supervised contrastive loss \citep{khosla2020supervised} and the balanced contrastive loss \citep{zhu2022balanced}, both of which continue to serve as widely used bases for analysis in long-tailed recognition \citep{cui2021parametric,cui2023generalized,pan2024enhanced, li2024long, wang2024llm, yi2025geometry,son2025difficulty, de2025long, wanshortcut, chen2026reframing}.
We show that the regular simplex is not a global minimizer of either objective which is nontrivial for BCL, which has been shown to admit such geometry at its minimum. Building on our framework, we propose a balancing method for supervised contrastive objectives and a fully balanced contrastive loss, a simple contrastive objective whose global minimizer is provably a regular simplex under long-tailed distributions.
\section{Preliminaries}\label{sec:prelim}
We consider a \(K\)-class classification task and learn a mapping from input data to their associated target labels. Let \(\mathcal{X}\) be an input space and \(X=(x_1, \dots,x_B)\in \mathcal{X}^B\) be a batch of \(B\) training samples. Denote the target space by \([K]=\{1,\dots,K\}\). Let \(Y=(y_1, \dots, y_B)\in [K]^B\) be the corresponding targets of the set \(X\). An encoder \(f:\mathcal{X}\to\mathcal{Z}\) maps an input to the representation space denoted by \(\mathcal{Z}\subset \mathbb{R}^h\), where \(h\) is the dimension of the representation. Let \(Z=(z_1, \dots, z_B)=(f(x_1),\dots,f(x_B))\) denote the set of representations obtained by applying \(f\) to \(X\).

In contrastive learning, each representation lies on a unit hypersphere \(\mathbb{S}^{h-1}=\{z\in \mathbb{R}^h: \|z\|=1\}\), where \(\|\cdot\|\) denotes the L2 norm. 
Following prior analyses of the optimal geometry of supervised contrastive objectives, we consider a mini-batch containing samples from all \(K\) classes. This assumption serves as an analytical setting for characterizing the geometry induced at the optimum, rather than requiring every practical mini-batch to contain all classes \citep{graf2021dissecting, zhu2022balanced, yi2025geometry}. In practical training, where this condition is generally infeasible for large \(K\), we follow prior work and adopt the class-complement setting, which additionally introduces class centers as contrastive pairs \citep{cui2021parametric, zhu2022balanced, cui2023generalized, son2025difficulty, yi2025geometry}. Accordingly, our theoretical results characterize the optimum under this full-class analytical setting, while our experiments evaluate structural balancing under practical stochastic mini-batch training.
Denote the index set of batch samples by \(\mathcal{B}=[B] = \{1,\dots, B\}\), and define a $k$-class subset as \(\mathcal{B}_k = \{i: i\in \mathcal{B}, y_i=k\}\) with cardinality \(|\mathcal{B}_k|=B_k\) for each class \(k\in[K]\). It follows that \(B = \sum_{k \in [K]} B_k\). We assume a batch with \(B>4\), $K\ge 3$ with augmentation, and a long-tailed distribution such that \(B_1 \ge \cdots \ge B_K\ge2\) with \(B_1 \gg B_K\). We now define contrastive losses used in our analysis.

\begin{defn}[SCL loss \citep{khosla2020supervised}]\label{sec:prelim_def_scl}
    Let \(Z=(z_1,\dots , z_B) \in (\mathbb{S}^{h-1})^B\) be a set of features with corresponding labels \(Y = (y_1 , \dots , y_B) \in [K]^B\). The SCL loss is defined by
    \begin{equation}
    \resizebox{.94\textwidth}{!}{$
    \displaystyle
        \mathcal{L}_{\mathrm{SCL}}(Z;Y,\mathcal{B})\! =\! \frac{1}{B}\! \sum_{y \in [K]}\sum_{i \in \mathcal{B}_y}
        \ell_{\mathrm{SCL}}^{(i)}, 
        \quad 
        \ell_{\mathrm{SCL}}^{(i)}\!=-\frac{1}{B_y - 1}\! \sum_{j \in \mathcal{B}_y \backslash \singleton{i}}\!\!\log\! \Bigg( \frac{\exp(z_i \!\cdot\! z_j/\tau)}{\sum\limits_{k \in \mathcal{B} \backslash \singleton{i}}\exp(z_i \! \cdot \! z_k/\tau)} \!\Bigg).\label{sec:prelim_def_scl-loss}
    $}
    \end{equation}
\end{defn}
\begin{defn}[BCL loss \citep{zhu2022balanced}]\label{sec:prelim_def_bcl}
    Let $Z=(z_1,\dots , z_B) \in (\mathbb{S}^{h-1})^B$ be the features with labels $Y = (y_1 , \dots , y_B) \in [K]^B$. 
    The BCL loss is defined by
    \begin{equation}
    \resizebox{.94\textwidth}{!}{$
    \displaystyle
        \mathcal{L}_{\mathrm{BCL}}(Z;Y, \mathcal{B})\! =\! \frac{1}{B} \! \sum_{y \in [K]}\sum_{i \in \mathcal{B}_y} \ell_{\mathrm{BCL}}^{(i)},
        \quad 
        \ell_{\mathrm{BCL}}^{(i)} \!= -
        \frac{1}{B_y - 1}\! \sum_{j \in \mathcal{B}_y \backslash \{i\}}\!\! \log\! \Bigg( \frac{\exp(z_i \!\cdot\! z_j/\tau)}{\sum\limits_{k \in [K]} \!\frac{1}{B_k}\! \sum\limits_{m \in \mathcal{B}_k}\!\exp(z_i \!\cdot\! z_m/\tau)} \!\Bigg),\label{sec:prelim_def_bcl-loss}
    $}
    \end{equation}
    where when \(k=y\), the anchor $i$ is excluded from the summation in the denominator, and $B_k$ is replaced by $B_y -1$.
\end{defn}
Following \citet{graf2021dissecting, zhu2022balanced, yi2025geometry}, we set the temperature parameter \(\tau=1\) in the losses for simplicity. 
\begin{defn}[Regular Simplex]\label{sec:prelim_def_regular}
Let $K, h \in \mathbb{N}$ satisfy $K \le h+1$. A \(K\)-point configuration $\zeta = (\zeta_1, \dots, \zeta_K) \in (\mathbb{S}^{h-1})^K$ forms a regular simplex if and only if
\begin{enumerate}[label=(\arabic*), itemsep=4pt, topsep=0pt, parsep=0pt]
    \item \(\forall i \in [K], \hspace{6pt} \|\zeta_i \|^2 = 1\). \label{def:regular_simplex_1}
    \item $\forall i,j\in[K]$ with $i\neq j$,\hspace{6pt} \(\zeta_i \cdot \zeta_j = -\tfrac{1}{K-1}\). \label{def:regular_simplex_2}
    \item \({\textstyle \sum_{i \in [K]}} \zeta_i = 0.\) \label{def:regular_simplex_3} 
\end{enumerate}
\end{defn}
We regard the regular simplex in Definition~\ref{sec:prelim_def_regular} as the optimal target geometry on the unit hypersphere: features collapsing to centers minimizes the within-class variance, while conditions~\ref{def:regular_simplex_2} and~\ref{def:regular_simplex_3} maximize the between-class variance. Thus, the Fisher discriminant ratio \citep{fisher1936use} is maximized, leading to an ideal geometry for classification. Also, it has been shown that yielding a regular simplex will improve generalization ability under long-tailed distributions \citep{liu2023inducing}.
\section{Representation Geometry under long-tailed distributions}
In this section, we analyze the geometry of the global minimizers of contrastive losses. Section~\ref{method-lim-scl-bcl} develops a framework that derives lower bounds for supervised contrastive objectives and uses them to examine regular simplex optimality. By showing that non-optimality at the lower bound implies non-optimality at the original loss, we prove that a regular simplex is not optimal for either of two representative losses, SCL or BCL. Section~\ref{method-sb} introduces structural balancing to overcome this failure and proposes Fully Balanced Contrastive Learning (FBCL), whose global minimizer is provably a regular simplex. Section~\ref{method-icv} analyzes the effect of structural balancing on intra-class variance penalty.

\subsection{Geometric Analysis of supervised contrastive objectives}\label{method-lim-scl-bcl}
Throughout this section, we refer to the original loss \(\mathcal{L}(Z;Y, \mathcal{B})\) defined over the feature configuration \(Z\) as the \textit{feature-level objective}. We derive a lower bound of \(\mathcal{L}\), denoted by \(\mathcal{LB}\), a function over the class centers \(\bar{Z}\). We refer to \(\mathcal{LB}(\bar{Z};K)\) as the \textit{class-level lower bound}. We derive the class-level lower bounds for contrastive objectives using Lemma~\ref{appendix-lem-anchor-wise} and Lemma~\ref{appendix-lem-class-wise} in the Appendix. Equality holds when features within each class collapse to their class center. Denote this condition by
\begin{enumerate}[label=(C\arabic*), series=LBconditions]
    \item \textbf{Class-collapse}: 
    \(\forall  y \in [K], \forall i \in \mathcal{B}_y:
    \; z_i = \bar{z}_y\). 
    \label{(C1)}
\end{enumerate}

\begin{prop}[Lower bound of the SCL loss]\label{prop-scl}
    Let \(Z=(z_1, \dots, z_B)\in(\mathbb{S}^{h-1})^B\) be a $B$-point configuration with labels \(Y=(y_1, \dots, y_B) \in [K]^B\). It holds that 
    \begin{align}
        \mathcal{L}_{\mathrm{SCL}}(Z;Y, \mathcal{B}) \ge \frac{1}{B}\! \sum_{y \in [K]} B_y \log \!\Big( B_y\!-\!1 + e^{-1}\!\sum\limits_{k \in [K]\backslash \singleton{y}} \!B_k \exp(\bar{z}_y \cdot \bar{z}_k) \!\Big) \eqcolon \mathcal{LB}_{\mathrm{SCL}}(\bar{Z};K) \label{lowerbound-scl}
    \end{align}
    where \(\bar{z}_y = \frac{1}{B_y}\sum_{i \in \mathcal{B}_y}z_i,\) and \(\bar{Z}=(\bar{z}_1, \dots, \bar{z}_K)\).
\end{prop}
\begin{prop}[Lower bound of the BCL loss]\label{prop-bcl}
    Let $Z, \bar{Z}$ be defined as in Proposition~\ref{prop-scl}. It holds that 
    \begin{align}
        \mathcal{L}_{\mathrm{BCL}}(Z;Y,\mathcal{B}) \ge \frac{1}{B}\! \sum_{y \in [K]} B_y \log \!\Big( 1 + e^{-1}\!\sum\limits_{k \in [K]\backslash \singleton{y}}  \!\exp(\bar{z}_y \cdot \bar{z}_k) \!\Big) \eqcolon \mathcal{LB}_{\mathrm{BCL}}(\bar{Z};K).
        \label{lowerbound-bcl}
    \end{align}
    Equality in the lower bounds in Eq.~\eqref{lowerbound-scl} and Eq.~\eqref{lowerbound-bcl} is attained if and only if \ref{(C1)} is satisfied.
\end{prop}
The proofs of Propositions mainly follow derivations from \citet{graf2021dissecting}, and are provided in Appendix~\ref{Prop-proof}. 

The class collapse condition alone does not specify the geometry of class centers. We introduce two Lemmas to characterize the geometry of the minimizers of a loss. Rather than analyzing the original loss, we analyze its lower bound. Lemma~\ref{lem-F-L} shows that if a center configuration is not optimal for the lower bound, then the corresponding feature configuration cannot be optimal for the original loss. Lemma~\ref{lem-minimizer_of_F} then derives a condition that the lower bound must satisfy for a regular simplex to be a minimizer. 
\begin{lem}\label{lem-F-L}
    Let \(\mathcal{L}\) be a feature-level contrastive loss,
    and let $\mathcal{LB}$ be a class-level lower bound of \(\mathcal{L}\). Let \(\zeta = (\zeta_1, \dots, \zeta_K) \in (\mathbb{S}^{h-1})^K\) be a $K$-point center configuration, and let \(Z^*=(z^*_1, \dots, z^*_B)\in (\mathbb{S}^{h-1})^B\) be a $B$-point configuration which collapses to \(\zeta\), i.e., \(z^*_i = \zeta_{y_i}\) for all \(i \in [B]\). Then,
    \begin{align}
        \zeta \notin \argmin_{\bar{Z}\in (\mathbb{S}^{h-1})^K} \mathcal{LB}(\bar{Z};K)
        \quad\Longrightarrow\quad
        Z^* \notin \argmin_{Z \in (\mathbb{S}^{h-1})^B} \mathcal{L}(Z;Y, \mathcal{B}).
    \end{align}
\end{lem}
\begin{lem}\label{lem-minimizer_of_F}
    Let $\bar{Z}, \bar{Z}^*$ be the center and the regular simplex configuration, respectively. $\mathcal{LB}(\bar{Z};K)$ can be written with coefficients $\beta, a_y,b_y,c_y$ for $y\in[K]$,
    \begin{align}
        \mathcal{LB}(\bar{Z};K) &= \frac{1}{\beta}\sum_{y \in [K]} a_y \log\Big(
            b_y + e^{-1}\sum_{m \in [K] \backslash \singleton{y}}c_m\exp(\bar{z}_y \cdot \bar{z}_m)
        \Big).
    \end{align}
    For $\bar{Z}^*$ to be a minimizer of $\mathcal{LB}(\bar{Z};K)$, the following condition must be satisfied: For every $y \in [K]$, with $d=-1/(K-1)$,
    \begin{align}
        \frac{a_y c_k}{b_y +e^{d-1}\sum_{m \in [K]\backslash\singleton{y}}c_m} + \frac{a_k c_y}{b_k + e^{d-1}\sum_{m \in [K]\backslash \singleton{k}}c_m}
        = \text{Constant}, \quad \text{for all } k \in[K]\backslash\singleton{y}.\label{eq:lem-minimizer_of_F}
    \end{align}
\end{lem}
The proofs are provided in Appendix~\ref{Appendix:lem-proof},~\ref{Thm-proof-derivation} respectively. These results form the basis of our framework: for a supervised contrastive objective, we derive a class-level lower bound \(\mathcal{LB}\), examine whether the regular simplex condition is satisfied, and when it is not, transfer this non-optimality to the original loss \(\mathcal{L}\). We apply this framework to SCL and BCL.
\begin{thm}\label{method:Thm-scl}
    Let \(\bar{Z}^* =(\bar{z}_1^*, \dots, \bar{z}_K^*)\) be a regular simplex, and let \(Z^*=(z^*_1, \dots, z^*_B)\) collapse class-wise to \(\bar{Z}^*\). Then, the regular simplex \(\bar{Z}^* \) is not a minimizer of \(\mathcal{LB}_{\mathrm{SCL}}(\bar{Z};K)\) by Lemma~\ref{lem-minimizer_of_F} and thus, by Lemma~\ref{lem-F-L}, \(Z^*\) is not a minimizer of \(\mathcal{L}_{\mathrm{SCL}}\). That is,
    \begin{align}
        Z^* \notin \argmin_{Z \in (\mathbb{S}^{h-1})^B} \mathcal{L}_{\mathrm{SCL}}(Z;Y, \mathcal{B}).\label{Thm-SCL-eq-F}
    \end{align}
\end{thm}
Theorem~\ref{method:Thm-scl} shows that the regular simplex is not a global minimizer of the SCL loss. 
The full statement of the theorem and detailed proof are provided in Appendix~\ref{Appendix:Thm-proof-scl}.
\begin{thm} \label{method:Thm-bcl}
    Let \(\bar{Z}^*\) and \(Z^*\) be defined as in Theorem~\ref{method:Thm-scl}. Then, \(\bar{Z}^*\) is not a minimizer of \(\mathcal{LB}_{\mathrm{BCL}}(\bar{Z};K)\) by Lemma~\ref{lem-minimizer_of_F} and thus, by Lemma~\ref{lem-F-L}, \(Z^*\) is not a minimizer of \(\mathcal{L}_{\mathrm{BCL}}\). That is,
    \begin{align}
         Z^* \notin \argmin_{Z \in (\mathbb{S}^{h-1})^B} \mathcal{L}_{\mathrm{BCL}}(Z;Y, \mathcal{B}).
         \label{Thm-BCL-eq-F}
    \end{align}
\end{thm}
Theorem~\ref{method:Thm-bcl} shows that the same result holds for BCL: although the BCL loss balances the contributions of negative samples within each anchor-wise loss, \(\mathcal{LB}_{\mathrm{BCL}}(\bar{Z};K)\) still retains \(B_y\), preventing Eq.~\eqref{eq:lem-minimizer_of_F} from being satisfied. The full statement and proof are provided in Appendix~\ref{Appendix:Thm-proof-bcl}.
\subsection{Fully Balanced Contrastive Learning via Structural Balancing}\label{method-sb}
Theorems~\ref{method:Thm-scl} and~\ref{method:Thm-bcl} show that the regular simplex is not a global minimizer of the representative losses SCL and BCL. We identify the reason for this failure and introduce a simple way to correct it.
Both losses balance the contribution of each class only within the anchor-wise loss \(\ell^{(i)}\): SCL loss normalizes over the positive set in \(\ell_{\mathrm{SCL}}^{(i)}\) (Eq.~\eqref{sec:prelim_def_scl-loss}), and BCL loss further normalizes contributions of negative samples of each class in the denominator of \(\ell_{\mathrm{BCL}}^{(i)}\) (Eq.~\eqref{sec:prelim_def_bcl-loss}). 

This anchor-wise balancing alone is insufficient: the anchor-wise losses are summed into an overall loss, which can be decomposed into two summations: \(\sum_{i \in \mathcal{B}} = \sum_{y \in [K]} \sum_{i \in \mathcal{B}_y}\). This summation weights the contribution of each class by its cardinality \(B_y\). This can be seen by rewriting the BCL loss as the following:
\begin{align}
    \mathcal{L}_{\mathrm{BCL}}(Z;Y,\mathcal{B})
    = \frac{1}{B}\!\sum_{i \in \mathcal{B}} \ell_{\mathrm{BCL}}^{(i)}
    = \frac{1}{B}\!\sum_{y \in [K]} \sum_{i \in \mathcal{B}_y} \ell_{\mathrm{BCL}}^{(i)}
    =\! \sum_{y \in [K]} \frac{B_y}{B}\Big( \frac{1}{B_y}\sum_{i \in \mathcal{B}_y} \ell_{\mathrm{BCL}}^{(i)}\Big). \label{sec:fbcl_bcl_structure}
\end{align}
Although the anchor-wise loss \(\ell^{(i)}_{\textrm{BCL}}\) is balanced, summing it over all samples assigns each class a weight of \(B_y/B\), as shown in Eq.~\eqref{sec:fbcl_bcl_structure}. Consequently, \(\mathcal{LB}_{\mathrm{BCL}}(\bar{Z};K)\) inherits the dependence on \(B_y\). Therefore, to remove this dependence, the summation of the anchors must be balanced.

We propose structural balancing, which extends balancing in contrastive objectives from the anchor-wise loss to the aggregation of these losses into the overall objective. While SCL and BCL balance within each anchor-wise loss, structural balancing balances the contributions of anchors across classes when forming the overall contrastive objective which differs in the motivation and usage from the conventional re-weighting methods. Details are provided in the Appendix~\ref{Appendix:Class-balanced_re-weighting}. We incorporate structural balancing into the BCL loss, leading to the following FBCL objective.
\begin{defn}[FBCL loss]\label{Def-FBCL}
    Let $Z=(z_1,\dots , z_B) \in (\mathbb{S}^{h-1})^B$ be features with labels $Y = (y_1 , \dots , y_B) \in [K]^B$. 
    The FBCL loss is defined by
    \begin{equation}
    \resizebox{0.94\textwidth}{!}{$
    \displaystyle
    \mathcal{L}_{\mathrm{FBCL}}(Z;Y, \mathcal{B})\! =\! \frac{1}{K}\! \sum_{y \in [K]}\! \frac{1}{B_y}\!\sum_{i \in \mathcal{B}_y} \ell_{\mathrm{FBCL}}^{(i)},
    \quad
    \ell_{\mathrm{FBCL}}^{(i)}\!
    =\!-
    \frac{1}{B_y - 1}\! \sum_{j \in \mathcal{B}_y \backslash \{i\}}\!\log\! \Bigg( \frac{\exp(z_i \!\cdot\! z_j/\tau)}{\sum\limits_{ k \in [K] }\! \frac{1}{B_k}\! \sum\limits_{m \in \mathcal{B}_k}\!\exp(z_i \!\cdot\! z_m/\tau)} \!\Bigg), \label{def-eq-fbcl}
    $}
    \end{equation}
    where \(\ell_{\mathrm{FBCL}}^{(i)} = \ell_{\mathrm{BCL}}^{(i)}\) and when \(k=y\), the anchor $i$ is excluded and $B_k$ is replaced by $B_y -1$.
\end{defn}
In Eq.~\eqref{def-eq-fbcl}, the factor \(1/K\) corresponds to \(1/B\) in the SCL and BCL losses, matching the scale.
We derive the class-level lower bound of FBCL to verify whether structural balancing removes the dependence on class cardinality at the lower bound.
\begin{prop}[Lower bound of FBCL]\label{prop-FBCL}
    Let \(Z, \bar{Z}\) be defined as in Proposition~\ref{prop-scl}. It holds that  
    \begin{align}
        \mathcal{L}_{\mathrm{FBCL}}(Z;Y,\mathcal{B}) \ge \frac{1}{K} \sum_{y \in [K]}\log \Big( 1 + e^{-1}\sum\limits_{k \in [K]\backslash \singleton{y}}  \exp(\bar{z}_y \cdot \bar{z}_k) \Big) \eqcolon \mathcal{LB}_{\mathrm{FBCL}}(\bar{Z};K)
        \label{lowerbound-fbcl}
    \end{align}
    The equality in Eq.~\eqref{lowerbound-fbcl} is attained if and only if \ref{(C1)} is satisfied.
\end{prop}
The proof is provided in Appendix~\ref{Prop-proof}. Proposition~\ref{prop-FBCL} shows that \(\mathcal{LB}_{\mathrm{FBCL}}(\bar{Z};K)\) assigns the same weight, \(1/K\), to every class and contains no coefficient of class cardinalities, removing the class dependence. As in Section~\ref{method-lim-scl-bcl}, we examine the configuration of class centers that minimizes \(\mathcal{LB}_{\mathrm{FBCL}}\).
\begin{thm}\label{method:Thm-fbcl}
    Let \(\bar{Z}^*\) and \(Z^*\) be defined as in Theorem~\ref{method:Thm-scl}. 
    We show that $\bar{Z}^*$ and $Z^*$ minimize \(\mathcal{LB}_{\mathrm{FBCL}}\) and \(\mathcal{L}_{\mathrm{FBCL}}\), respectively.
    \begin{align}
         Z^* \in \argmin_{Z \in (\mathbb{S}^{h-1})^B} \mathcal{L}_{\mathrm{FBCL}}(Z;Y,\mathcal{B}).
         \label{Thm-FBCL-eq-F}
    \end{align}
    In particular, the minimum of the FBCL loss is given by
    \begin{align}
        \mathcal{L}_{\mathrm{FBCL}}(Z;Y,\mathcal{B}) \ge \log \Big( 1 + (K-1)\exp\big(-\tfrac{K}{K-1} \big) \Big),
    \end{align}
    where equality is attained if and only if the representation is in the form of \(Z^*\), a regular simplex.
\end{thm}
The full statement of the theorem and its proof are provided in Appendix~\ref{Appendix:Thm-proof-fbcl}. Theorem~\ref{method:Thm-fbcl} characterizes the geometric configuration of \(Z\) at the minimum of \(\mathcal{L}_{\mathrm{FBCL}}\). It shows that \(\mathcal{L}_{\mathrm{FBCL}}\) is minimized if and only if the representations collapse to class centers forming a regular simplex.
\subsection{Analysis of intra-class variance penalty}\label{method-icv}
We further analyze the effect of structural balancing on intra-class variance. Following \citet{wang2020understanding}, we decompose each loss as \(\mathcal{L}= \mathcal{L}^{\mathrm{alignment}} + \mathcal{L}^{\mathrm{uniformity}}\),
where the alignment term encourages similarity between positive pairs. Full decompositions are provided in Appendix~\ref{appendix-decomposition}.
Since SCL and BCL differ only in their denominator, they share the same alignment term. Let
\(V_y= \frac{1}{B_y}\sum_{i \in \mathcal{B}_y} \| z_i - \bar{z}_y \|^2\)
denote the intra-class variance of class \(y\). Then,
\begin{align}
    \resizebox{0.90\textwidth}{!}{$
    \displaystyle
    \mathcal{L}^{\text{alignment}}_{\mathrm{SCL}}
    =
    \mathcal{L}^{\text{alignment}}_{\mathrm{BCL}}
    =
    -\frac{1}{B}\sum_{y \in [K]} \sum_{i\in \mathcal{B}_y}
    \frac{1}{B_y-1}\sum_{j\in \mathcal{B}_y\backslash \singleton{i}} (z_i \cdot z_j)
    =
    \sum_{y \in [K]} \frac{B_y^2}{B(B_y-1)}V_y -1.
    $}
\end{align}
The coefficient of \(V_y\) satisfies
\(\frac{B_y^2}{B(B_y-1)} \approx \frac{B_y}{B}\),
so the variance penalty is approximately proportional to class cardinality, resulting in stronger contraction for head classes.
After structural balancing, the alignment term becomes
\begin{align}
    \resizebox{0.90\textwidth}{!}{$
    \displaystyle
    \mathcal{L}^{\text{alignment}}_{\mathrm{FBCL}}
    =
    -\frac{1}{K}\!\sum_{y \in [K]}\! \frac{1}{B_y}\!\sum_{i\in \mathcal{B}_y}\!
    \frac{1}{B_y-1}\!\!\sum_{j\in \mathcal{B}_y\backslash \singleton{i}}
    \!\!(z_i \cdot z_j) 
    =
    \sum_{y \in [K]} \frac{B_y}{K(B_y-1)}V_y -1.
    $}
\end{align}
Here,
\(\frac{B_y}{K(B_y-1)} \!\!\approx \!\!\frac{1}{K}\),
which is approximately independent of class cardinality. 
Thus, within the alignment term, structural balancing makes the penalty on intra-class variance nearly uniform across classes.

\section{Experiments}
We conduct experiments on standard long-tailed datasets. We also conduct experiments on diverse downstream datasets to evaluate not only classification performance but also the transferability of representations learned under class imbalance, an aspect that has not been widely reported.
\subsection{Datasets}
\paragraph{Long-tailed recognition} 
We use ImageNet-LT \citep{liu2019large}, a long-tailed subset of ImageNet \citep{deng2009imagenet} containing 115.8K images from 1,000 classes (cardinalities ranging from 5 to 1,280), and iNaturalist \citep{van2018inaturalist}, a large-scale imbalanced dataset of 437.5K images from 8,142 classes. We report accuracy on many-shot ($>$100 images), medium-shot (20 $\sim$ 100), few-shot ($<$20), and all classes.
\paragraph{Transfer learning}
To evaluate the transferability of the learned representations, we run experiments on the following datasets: CIFAR-10/100 \citep{krizhevsky2009learning}, MIT67 \citep{quattoni2009recognizing}, DTD \citep{cimpoi2014describing}, Oxford-IIIT Pets \citep{parkhi2012cats}, Caltech101 \citep{fei2004learning}, Flowers \citep{nilsback2008automated}, CUB200 \citep{wah2011caltech}, Food \citep{bossard2014food}, Stanford Cars \citep{krause20133d} and FGVC Aircraft \citep{maji2013fine}. We report results under both linear evaluation and fine-tuning.
\subsection{Dual-branch Training Framework}
We use the dual-branch framework, which is commonly used to learn better representations under long-tailed distributions \citep{kang2019decoupling, wang2021contrastive, zhu2022balanced, suh2023long, du2024probabilistic, yi2025geometry}. The model consists of a base encoder \(\phi:\mathcal{X} \to \mathcal{H}\), a prediction head \(\psi: \mathcal{H} \to \mathbb{R}^K\), and a contrastive head \(g: \mathcal{H}\to \mathcal{Z}\). 
Given an input $(x,y)$, the model produces $h=\phi(x)$, class logits $p=\psi(h)$, and a normalized contrastive feature $z=g(h)$.
The prediction head is optimized using a logit-adjusted cross-entropy loss \(\mathcal{L}_{LC}\) \citep{menon2021longtail}:
\begin{align}
    \mathcal{L}_{LC}(x,y) &= -\log \Bigg( \frac{\exp(p_y + \delta_y)}{\sum_{y' \in [K]}\exp(p_{y'} + \delta_{y'})} \Bigg),
\end{align}
where \(\pi_y\) is the empirical class prior of class \(y\), and \(\delta_y = \log \pi_y\) is the logit adjustment term.
The contrastive head is optimized using a contrastive loss, \(\mathcal{L}_{CL}\). Overall, the training objective is:
\begin{align}\label{sec:Exp_train_objective}
    \mathcal{L} = \lambda \mathcal{L}_{LC} + \mu \mathcal{L}_{CL}
\end{align}
where \(\lambda\) and \(\mu\) are hyperparameters controlling the contributions of the two objectives.
\begin{table}[t]
\renewcommand{\thetable}{2}
\centering
\caption{Comparison of long-tailed recognition methods and the effect of structural balancing on ImageNet-LT and iNaturalist. \(^*\) denotes results reproduced with 90 epochs on ImageNet-LT and 100 epochs on iNaturalist for a fair comparison. \textbf{Bold} indicates the best results. Standard deviations are reported in Appendix~\ref{Appendix-statistics}.}
\label{table:main_results}
\resizebox{0.9\linewidth}{!}{%
\begin{tabular}{@{}l|cccc|cccc@{}}
\toprule
\multirow{2}{*}{\textbf{Method}} & \multicolumn{4}{c|}{ImageNet-LT} & \multicolumn{4}{c}{iNaturalist} \\ \cmidrule(l){2-9} 
& Many   & Medium  & Few   & All   & Many  & Medium & Few   & All  \\ \midrule
OLTR \citep{liu2019large}                & 43.2   & 35.1    & 18.5  & 35.6  & 59.0 & 64.1 & 64.9 & 63.9 \\
\(\tau\)-norm \citep{kang2019decoupling} & 56.6   & 44.2    & 27.4  & 46.7  & 65.6 & 65.3 & 65.9 & 65.6 \\
cRT \citep{kang2019decoupling}           & 58.8   & 44.0    & 26.1  & 47.3  & 69.0 & 71.1 & 70.2 & 70.6 \\
PaCo\(^*\) \citep{cui2021parametric}     & 55.3   & 50.2    & 26.7  & 48.9  & 60.2 & 67.7 & 68.5 & 67.2 \\
GPaCo\(^*\) \citep{cui2023generalized}   & 59.4   & 53.4    & 30.5  & 52.6  & 65.6 & 72.2 & 72.0 & 71.4 \\
LA \citep{menon2021longtail}             & 62.8   & 49.0    & 31.8  & 52.0  & 68.4 & 69.4 & 69.2 & 69.2 \\
DisAlign \citep{zhang2021distribution}   & 61.3   & 52.2    & 31.4  & 52.9  & 61.6 & 70.8 & 69.9 & 69.5 \\
Hybrid-PSC \citep{wang2021contrastive}   & -      & -       & -     & -     & -    & -    & -    & 68.1 \\
MiSLAS \citep{zhong2021improving}        & 61.7   & 51.3    & 35.8  & 52.7  & \textbf{73.3} & 72.4 & 70.4 & 71.6 \\
KCL \citep{kang2020exploring}            & 61.8   & 49.4    & 30.9  & 51.5  & -    & -    & -    & 68.6 \\
TSC \citep{li2022targeted}               & 63.5   & 49.7    & 30.4  & 52.4  & 72.6 & 70.6 & 67.8 & 69.7 \\
SBCL \citep{hou2023subclass}             & 63.8   & 51.3    & 31.2  & 53.4  & \textbf{73.3} & 71.9 & 68.6 & 70.8 \\
Focal-SAM \citep{li2025focal}            & 63.9   & 52.2    & 34.4  & 54.3  & 68.4 & 72.0 & 72.5 & 71.8 \\
BCL\(^*\) \citep{zhu2022balanced}        & 65.8   & 53.7    & 36.2  & 56.0  & 70.0 & 72.0 & 71.1 & 71.3 \\
\rowcolor{gray!13} BCL+SB                & 66.4   & 54.1    & 37.5  & 56.6  & 70.4 & 72.8 & 71.9 & 72.2 \\
DBM+BCL$^*$  \citep{son2025difficulty}   & 64.4   & 54.3    & 41.4  & 56.5  & 66.3 & 71.8 & 73.3 & 71.8 \\
\rowcolor{gray!13} DBM+BCL+SB            & 64.9   & 54.5    & \textbf{42.3}  & 56.9  & 66.5 & 72.4 & \textbf{73.8}  & 72.3 \\
ProCo$^*$ \citep{du2024probabilistic}    & 66.5   & 54.8    & 38.1  & 57.1  & 70.4 & 73.1 & 71.9 & 72.3 \\
\rowcolor{gray!13} ProCo+SB              & \textbf{66.8}   & \textbf{55.2}    & 38.8  & \textbf{57.4}  & 70.7 & \textbf{73.4} & 72.3 & \textbf{72.7} \\ \bottomrule
\end{tabular}}
\end{table}
\begin{table}[t]
\renewcommand{\thetable}{3}
\caption{Comparison of transferability across 11 natural image classification datasets. All backbones are pretrained on ImageNet-LT with 90 epochs of pretraining for a fair comparison. \(^\dagger\) denotes results reproduced using the training framework in Eq.~\eqref{sec:Exp_train_objective}, replacing the contrastive objective. \textbf{Bold} and \underline{underline} indicate the best and second-best results, respectively.}
\label{table:transferability}
\centering
\scriptsize
\setlength{\tabcolsep}{3.5pt}
\makebox[\linewidth][c]{
\resizebox{0.92\linewidth}{!}{
\begin{tabular}{@{}lccccccccccc|c@{}}
\toprule
\textbf{Method} & CIFAR10 & CIFAR100 & MIT67 & DTD & Pets & Caltech101 & Flowers & CUB200 & Food & Cars & Aircraft & Mean \\ 
\midrule
\multicolumn{13}{@{}l}{\textit{Linear evaluation}} \\
\midrule
Cross-Entropy & 86.8 & 65.4 & 60.5 & 63.5 & 80.1 & 83.1 & 84.6 & 49.2 & 60.5 & 35.2 & \multicolumn{1}{c|}{37.4} & 64.2 \\
SimCLR \citep{chen2020simple} & 83.0 & 61.2 & 57.8 & 64.4 & 57.2 & 69.9 & 78.7 & 24.6 & 58.2 & 28.4 & \multicolumn{1}{c|}{36.5} & 56.4 \\
PaCo \citep{cui2021parametric} & 86.2 & 63.8 & 60.9 & 63.3 & 82.3 & 83.5 & 83.5 & 50.4 & 60.8 & 37.1 & 39.2 & 64.6 \\
GPaCo \citep{cui2023generalized} & \textbf{88.0} & 67.0 & 62.8 & 65.5 & 84.5 & 85.8 & 83.6 & 50.9 & 61.2 & 38.9 & 40.8 & 66.3 \\
\midrule
SimCLR\(^\dagger\) \citep{chen2020simple} & 86.8 & \underline{67.3} & 64.6 & 66.0 & 84.9 & 87.1 & \underline{88.6} & \textbf{55.6} & \underline{63.2} & \underline{45.4} & 43.8 & 68.5 \\
SCL\(^\dagger\) \citep{khosla2020supervised} & 87.2 & 67.2 & \textbf{66.5} & 65.6 & 84.1 & \underline{88.1} & 88.3 & 54.8 & \textbf{63.5} & 43.4 & \multicolumn{1}{c|}{\underline{45.3}} & 68.6 \\
BCL\(^\dagger\) \citep{zhu2022balanced} & 86.8 & 66.5 & 65.7 & \underline{66.2} & \textbf{85.9} & \underline{88.1} & \textbf{88.9} & 54.9 & 62.8 & 45.1 & \multicolumn{1}{c|}{45.0} & \underline{68.7} \\
\rowcolor{gray!13} FBCL\(^\dagger\) (Ours) & \underline{87.4} & \textbf{67.5} & \underline{66.1} & \textbf{66.4} & \underline{85.7} & \textbf{88.3} & 88.4 & \underline{55.5} & 62.9 & \textbf{46.0} & \multicolumn{1}{c|}{\textbf{45.5}} & \textbf{69.1} \\ \midrule
\multicolumn{13}{@{}l}{\textit{Fine-tuning}} \\\midrule
Cross-Entropy & 91.7 & 80.2 & 66.0 & 59.9 & 80.2 & 81.6 & 83.7 & 57.7 & 84.6 & 80.0 & \multicolumn{1}{c|}{74.2} & 76.3 \\
SimCLR \citep{chen2020simple} & 96.1 & 78.4 & 67.0 & 65.2 & 81.3 & 76.6 & 90.3 & 67.2 & 85.7 & 88.3 & \multicolumn{1}{c|}{79.9} & 79.6 \\
PaCo \citep{cui2021parametric} & 93.4 & 78.8 & 71.6 & 67.6 & 87.1 & 88.0 & 93.9 & 73.2 & \textbf{86.1} & \textbf{90.7} & 80.1 & 82.7 \\
GPaCo \citep{cui2023generalized} & \textbf{96.8} & \textbf{82.6} & \textbf{73.4} & 69.9 & 86.9 & \underline{90.0} & \underline{95.1} & \underline{75.0} & \underline{85.9} & \underline{90.2} & \underline{81.9} & \underline{84.3} \\
\midrule
SimCLR\(^\dagger\) \citep{chen2020simple} & \underline{96.3} & 81.1 & 69.3 & \underline{70.3} & 87.3 & 88.4 & \underline{95.1} & 73.1 & 85.5 & 90.1 & 81.5 & 83.5 \\
SCL\(^\dagger\) \citep{khosla2020supervised} & 96.0 & 82.0 & 71.0 & 69.4 & 87.0 & 88.7 & 94.7 & 73.1 & 85.1 & \underline{90.2} & \multicolumn{1}{c|}{80.8} & 83.5 \\
BCL\(^\dagger\) \citep{zhu2022balanced} & 96.0 & \underline{82.3} & 69.0 & \underline{70.3} & \underline{88.4} & 89.6 & 94.9 & \textbf{75.3} & 85.2 & 89.5 & \multicolumn{1}{c|}{81.6} & 83.8 \\
\rowcolor{gray!13} FBCL\(^\dagger\) (Ours) & \underline{96.3} & \textbf{82.6} & \underline{72.4} & \textbf{71.0} & \textbf{88.5} & \textbf{90.5} & \textbf{95.2} & 74.6 & 85.6 & \underline{90.2} & \multicolumn{1}{c|}{\textbf{82.8}} & \textbf{84.5}\\ \bottomrule
\end{tabular}}}
\end{table}
\begin{wraptable}[8]{r}{0.48\linewidth}
\centering{
\renewcommand{\thetable}{1}
\caption{Comparison of classification accuracy across contrastive objectives on ImageNet-LT.}\label{table:contrastive_comparison}
\resizebox{\linewidth}{!}{%
\begin{tabular}{@{}l|cccc@{}}
\toprule
\multirow{2}{*}{\textbf{Method}} & \multicolumn{4}{c}{ImageNet-LT} \\ \cmidrule(l){2-5} 
 & Many & Medium & Few & All \\ \midrule
SimCLR \citep{chen2020simple} & 65.0 & 52.6 & 36.2 & 55.2 {\small$\pm$ 0.2}\\
SCL \citep{khosla2020supervised} & 65.8 & 53.3 & 36.1 & 55.8 {\small$\pm$ 0.4}\\
BCL \citep{zhu2022balanced} & 65.8 & 53.7 & 36.2 & 56.0 {\small$\pm$ 0.1} \\
\rowcolor{gray!13} FBCL (Ours) & \textbf{66.4} & \textbf{54.1} & \textbf{37.5} & \textbf{56.6} {\small$\pm$ 0.1} \\ \bottomrule
\end{tabular}}}
\end{wraptable}
\subsection{Implementation Details}
Our implementation follows \citet{zhu2022balanced}. 
We use ResNet-50 \citep{he2016deep} as the model backbone, and use RandAug \citep{cubuk2020randaugment} for data augmentation. To compare FBCL with different contrastive objectives within the same training framework, we report the results with \(\mathcal{L}_{CL}\) in Eq.~\eqref{sec:Exp_train_objective} set to SimCLR, SCL, and BCL losses. For a fair comparison, the dimension of the contrastive head is 1024 for both ImageNet-LT and iNaturalist, and 128 for CIFAR-100-LT, following \citep{zhu2022balanced}. 
Detailed hyperparameter settings are provided in Appendix~\ref{Appendix-implementation}. For transfer learning, we first train the backbone on ImageNet-LT and then perform linear evaluation and full fine-tuning on each dataset. Details on transfer learning are provided in Appendix~\ref{Appendix-eval-protocol}.
\begin{table}[t]
\centering{
\renewcommand{\thetable}{4}
\caption{Comparison of intra-class variance and Top-1 accuracy. (a) Results on synthetic data. (b) Ablation study of positive balancing (PB), negative balancing (NB), and structural balancing (SB) on ImageNet-LT. Details are provided in Appendix~\ref{Appendix-statistics}.}
\vspace{-8pt}
\label{table:icv}
\begin{subtable}[t]{0.49\linewidth}
\centering
\caption{Synthetic data}
\label{table:icv-synthetic}
\renewcommand{\arraystretch}{0.97}
\resizebox{\linewidth}{!}{%
\begin{tabular}{l|ccc|c} \toprule
\multirow{2}{*}{\textbf{Method}} & \multicolumn{3}{c|}{\textbf{Intra-class variance}} & \multirow{2}{*}{\textbf{Acc.}} \\
\cmidrule(lr){2-4} & Many & Medium & Few & \\ \midrule
CE & 0.459 & 0.533 & 0.721 & 73.8 {\tiny$\pm 0.16$} \\ \midrule
SCL & 0.272 & 0.346 & 0.575 & 74.1 \tiny{$\pm 2.74$} \\
BCL & 0.250 & 0.385 & 0.749 & 74.1 {\tiny$\pm 2.36$} \\
\rowcolor{gray!13} FBCL (Ours) & 0.409 & 0.453 & 0.438 & 75.5 {\tiny$\pm 3.10$} \\ \bottomrule
\end{tabular}}
\end{subtable}\hspace{0.005\linewidth}%
\begin{subtable}[t]{0.49\linewidth}
\centering
\caption{ImageNet-LT}
\label{table:icv-imagenet}
\renewcommand{\arraystretch}{1.29}
\resizebox{\linewidth}{!}{%
\begin{tabular}{@{}l|ccc|ccc|c@{}} \toprule
\multirow{2}{*}{\textbf{Method}} & \multirow{2}{*}{\textbf{PB}} & \multirow{2}{*}{\textbf{NB}} & \multirow{2}{*}{\textbf{SB}} & \multicolumn{3}{c|}{\textbf{Intra-class variance}} & \multirow{2}{*}{\textbf{Acc.}} \\
\cmidrule(lr){5-7} & & & & Many & Medium & Few & \\\midrule
PB (SCL) & \cmark & \xmark & \xmark & 0.189 & 0.298 & 0.357 & 55.78 {\small$\pm 0.41$} \\
PB+NB (BCL) & \cmark & \cmark & \xmark & 0.167 & 0.293 & 0.362 & 56.03 {\small$\pm 0.07$} \\ \midrule
PB+SB & \cmark & \xmark & \cmark & 0.368 & 0.502 & 0.509 & 55.92 {\small$\pm 0.19$} \\
NB+SB & \xmark & \cmark & \cmark & 0.597 & 0.623 & 0.602 & 55.35 {\small$\pm 0.17$} \\ \rowcolor{gray!13}
PB+NB+SB (Ours) & \cmark & \cmark & \cmark & 0.196 & 0.281 & 0.339 & 56.64 {\small$\pm 0.12$} \\ \bottomrule
\end{tabular}}
\end{subtable}}
\end{table}
\subsection{Results}
\vspace{-5pt}
Table~\ref{table:main_results} reports the top-1 accuracy of long-tailed recognition methods and methods combined with structural balancing on ImageNet-LT and iNaturalist. Structural balancing consistently improves over the base methods, showing its applicability and that it can be readily incorporated into other methods. Table~\ref{table:transferability} reports the transferability results, with FBCL achieving the highest average performance on both evaluations. Tables~\ref{table:contrastive_comparison} and~\ref{table:main_results} show that FBCL achieves the highest results among contrastive methods based on SCL or BCL \citep{wang2021contrastive,cui2021parametric,cui2023generalized,kang2020exploring,li2022targeted,zhu2022balanced,hou2023subclass}. Results with other methods, datasets and standard deviations are provided in Appendix~\ref{Appendix-statistics}.
\vspace{-8pt}
\subsection{Ablation Study}
\vspace{-5pt}
To examine the individual role of each balancing method in the FBCL loss, we conduct an ablation study on ImageNet-LT. As described in Section~\ref{method-sb}, FBCL contains three balancing components: positive balancing (PB), negative balancing (NB), and structural balancing (SB). Please refer to Appendix~\ref{Appendix-balancing_components} for each balancing component. From Table~\ref{table:icv-imagenet}, we first observe that SB reduces the disparity of intra-class variance across classes. Without SB, the intra-class variance gap between many and few classes is 0.168 for SCL and 0.195 for BCL, while adding SB reduces these gaps to 0.141 for PB+SB and 0.143 for FBCL. 
The results of NB+SB further show that reducing the variance gap alone is not sufficient, and that lowering the intra-class variance within each class group is also important. Comparing NB+SB with FBCL isolates the role of PB: adding PB reduces intra-class variance across all groups while improving overall accuracy. Likewise, comparing PB+SB with FBCL shows that removing NB decreases accuracy, indicating the contribution of NB. Together, these results show that PB, NB, and SB are complementary.
\section{Discussion}\label{sec:discussion}
Our theoretical analysis follows the standard setting adopted in prior geometric analyses of contrastive learning, where each batch contains all classes \citep{graf2021dissecting,zhu2022balanced,yi2025geometry}. This setting is used to characterize the geometry of the objective at its optimum, and such analyses still provide a principled indication of the geometric tendency induced by the objective under practical mini-batch training. Within this analytical setting, our contribution is to show that the regular simplex guarantee does not hold for previous methods under long-tailed distributions, and to establish such a guarantee for FBCL. The dimensional condition is satisfied on ImageNet-LT but not on iNaturalist. The attainability of such geometric optima in large-scale training is beyond the scope of this work and remains underexplored even for contrastive objectives under large-scale balanced distributions \citep{graf2021dissecting}; how these geometric tendencies manifest under partial class coverage in mini-batches or when \(K>h+1\) provides an interesting direction for future study. Nevertheless, structural balancing itself is not restricted to these analytical conditions and consistently improves the methods with which it is combined in practical large-scale settings, including iNaturalist. Thus, our theory characterizes how balancing changes the objective and its optimum under the same analytical regime used in prior works, while the experiments show that the resulting mechanism remains effective under practical stochastic training.
\section{Conclusion}
In this paper, we studied the representation geometry induced by supervised contrastive objectives under long-tailed distributions and developed a framework for analyzing regular simplex optimality. We showed that the regular simplex is not a global minimizer of SCL or BCL under class imbalance due to the remaining dependence on class cardinality in the aggregation of anchor-wise losses. To remove this dependence, we introduced structural balancing and proposed Fully Balanced Contrastive Learning (FBCL), which admits a regular simplex as a global optimum. We further showed that structural balancing makes the alignment term penalty on intra-class variance approximately independent of class frequency, and empirically observed more balanced intra-class variances across classes. Experiments on ImageNet-LT, iNaturalist, and transfer tasks demonstrate improvements in long-tailed recognition and representation transfer, as well as the effectiveness of structural balancing when combined with other methods. Our work thus provides a principled way to analyze and improve supervised contrastive learning under long-tailed distributions.

\bibliography{iclr2027_conference}
\bibliographystyle{iclr2027_conference}
\newpage
\appendix
\section{Appendix}
\subsection{Preliminaries}
We summarize and add details of the preliminaries introduced in Section~\ref{sec:prelim}.
\begin{itemize}[leftmargin=*, itemsep=0pt, topsep=0pt, partopsep=0pt]
\item We adopt the setting that a mini-batch contains all classes, as is commonly assumed when analyzing the optimal geometry \citep{graf2021dissecting, zhu2022balanced, yi2025geometry}.
\item Let $\mathcal{B}$ be an index set of an arbitrary batch.
\item \(K\): number of classes, assume \(K \ge 3\).
\item $\mathcal{B} = \dot{\bigcup}_{y \in [K]} \mathcal{B}_y$. And thus, $|\mathcal{B}|=\sum_{y \in [K]}|\mathcal{B}_y|$.
\item Let $|\mathcal{B}|=B>4$, $|\mathcal{B}_y| = B_y$ for any $y \in [K]$. As we augment samples for contrastive pairs, we assume $B_y\ge2$. 
\item Long-tailed distribution: \(B_1 \ge B_2 \ge \cdots \ge B_K\) with \(B_1 \gg B_K\).
\item Let \(\mathbb{S}^{h-1} = \{z \in \mathbb{R}^{h}: \| z \|^2 = 1\}\) denote the unit sphere.
\item Following \citet{graf2021dissecting},\citet{zhu2022balanced}, and \cite{yi2025geometry}, the temperature parameter is fixed to \(\tau = 1\).
\end{itemize}

\section{Theoretical framework and proofs}\label{Thm-proof}
We establish our framework in three steps that can be used for any contrastive objective. Then, we apply our framework to representative objectives under a long-tailed distribution. 

For a given supervised contrastive objective $\mathcal{L}$, Lemmas~\ref{appendix-lem-anchor-wise} and~\ref{appendix-lem-class-wise} derive $\mathcal{LB}$ of $\mathcal{L}$, a lower bound whose equality condition corresponds to class-collapse~\ref{(C1)-appendix}.

Results of Lemmas~\ref{appendix-lem-F-L} and~\ref{appendix-lem:unified-lagrangian} correspond to consistency and characterization, respectively. Concretely, we analyze the loss through its lower bound which depends on the number of classes. To ensure that the results derived from the lower bound are consistent with those from the original loss, that is, that the geometry of the representations is consistent under both analyses, we establish Lemma~\ref{appendix-lem-F-L}. Therefore, if a particular geometry configuration does not minimize the lower bound, it also does not minimize the original loss. If it does, we further derive a lower bound of $\mathcal{LB}$ whose equality condition characterizes the optimal geometry. Lemma~\ref{appendix-lem:unified-lagrangian} allows us to analyze a lower bound through its coefficients and characterize whether a regular simplex configuration minimizes it.

With these Lemmas, our framework is as follows: 
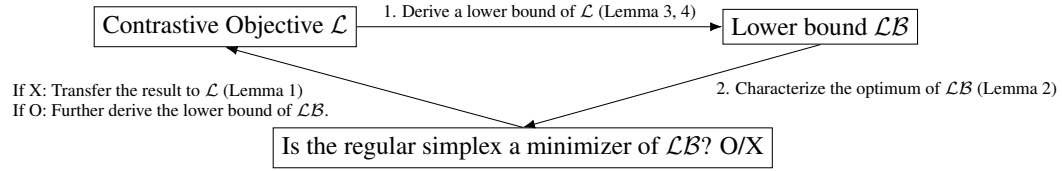
\begin{figure}[h]
    \centering
    \resizebox{\linewidth}{!}{%
    \begin{tikzpicture}
      \node[draw] (node1) at (-2.80,0.67) {Contrastive Objective $\mathcal{L}$};
      \node[draw] (node2) at (5.4,0.67) {Lower bound $\mathcal{LB}$};
      \draw[arrows=-Latex] (node1.east) --
                            node[above, sloped=false, node font=\scriptsize] {1. Derive a lower bound of $\mathcal{L}$ (Lemma 3, 4)}
                            (node2.west);
      \node[draw] (node3) at (1.3,-1) {Is the regular simplex a minimizer of $\mathcal{LB}$? O/X};
      \draw[arrows=-Latex] (node2.south) --
                            node[sloped=false, node font=\scriptsize, below right, pos=0.37] {2. Characterize the optimum of $\mathcal{LB}$ (Lemma 2)}
                            (node3.north);
      \draw[arrows=-Latex] (node3.north) --
                            node[sloped=false, node font=\scriptsize, below left, pos=0.63, align=left] {If X: Transfer the result to $\mathcal{L}$ (Lemma 1)\\If O: Further derive the lower bound of $\mathcal{LB}$.}
                            (node1.south);
    \end{tikzpicture}}\caption{Schematic overview of our framework.}
\end{figure}

Leveraging this framework, we investigate the optimality of SCL, BCL, and FBCL in Theorems~\ref{Appendix:Thm-SCL},~\ref{Appendix:Thm-BCL}, and~\ref{Appendix:Thm-FBCL}.

\subsection{Reduction from feature-level to class-level}\label{Appendix:lem-proof}

The following lemma establishes the consistency between the two minimization problems when a geometry fails to minimize the lower bound: a class-collapsed feature configuration fails to be a global minimizer of \(\mathcal{L}\) if its center configuration fails to be a minimizer of $\mathcal{LB}$. This consistency allows us to conduct the analysis at the lower bound and determine when a regular simplex fails to be a global minimizer.
\begin{applem}\label{appendix-lem-F-L}
    Let \(\mathcal{L}\) be a feature-level contrastive loss,
    and let $\mathcal{LB}$ be a class-level lower bound of \(\mathcal{L}\). Let \(\zeta = (\zeta_1, \dots, \zeta_K) \in (\mathbb{S}^{h-1})^K\) be a $K$-point center configuration, and let \(Z^*=(z^*_1, \dots, z^*_B)\in (\mathbb{S}^{h-1})^B\) be a $B$-point configuration which collapses to \(\zeta\), i.e., \(z^*_i = \zeta_{y_i}\) for all \(i \in [B]\). Then,
    \begin{align}
        \zeta \notin \argmin_{\bar{Z}\in (\mathbb{S}^{h-1})^K} \mathcal{LB}(\bar{Z};K)
        \quad\Longrightarrow\quad
        Z^* \notin \argmin_{Z \in (\mathbb{S}^{h-1})^B} \mathcal{L}(Z;Y, \mathcal{B}).
    \end{align}
\end{applem}
\begin{proof}
    For brevity, we omit the dependence on \(Y\) and $\mathcal{B}$ and write 
    \(\mathcal{L}(Z)\) in place of \(\mathcal{L}(Z;Y,\mathcal{B})\).
    Suppose that
    \begin{align}
        \zeta \notin \argmin_{\bar Z \in (\mathbb{S}^{h-1})^K} \mathcal{LB}(\bar Z).
    \end{align}
    Then, there exists 
    \(\tilde{\zeta}=(\tilde{\zeta}_1,\dots,\tilde{\zeta}_K)
    \in(\mathbb{S}^{h-1})^K\) such that
    \begin{align}
        \mathcal{LB}(\tilde{\zeta}) < \mathcal{LB}(\zeta).
    \end{align}
    Define the class-collapsed feature configuration 
    \(\tilde Z=(\tilde z_1,\dots,\tilde z_B)\in(\mathbb{S}^{h-1})^B\) by
    \begin{align}
        \tilde z_i=\tilde\zeta_{y_i}
        \quad \text{for each } i\in[B].
    \end{align}
    Since \(\tilde Z\) collapses to \(\tilde\zeta\) and \(Z^*\) collapses to \(\zeta\), 
    the equality of the lower bound under class collapse gives
    \begin{align}
        \mathcal{L}(\tilde Z)=\mathcal{LB}(\tilde\zeta),
        \qquad
        \mathcal{L}(Z^*)=\mathcal{LB}(\zeta).
    \end{align}
    Therefore,
    \begin{align}
        \mathcal{L}(\tilde Z)
        =
        \mathcal{LB}(\tilde\zeta)
        <
        \mathcal{LB}(\zeta)
        =
        \mathcal{L}(Z^*).
    \end{align}
    Hence, \(Z^*\) cannot be a global minimizer of \(\mathcal{L}\), and thus
    \begin{align}
        Z^* \notin \argmin_{Z \in (\mathbb{S}^{h-1})^B}
        \mathcal{L}(Z;Y,\mathcal{B}).
    \end{align}
\end{proof}
Moreover, \textbf{we give another version} of Lemma~\ref{appendix-lem-F-L}, which we do not use in this work, \textbf{but which may be useful for future studies}. This Lemma gives an if and only if condition of the optimality between the class-level lower bound and the original loss. The only added assumption is the dimension condition, $K<h$.
\addtocounter{applem}{-1}
\begin{applem}
    Let \(\mathcal{L}\) be a feature-level contrastive loss,
    and let $\mathcal{LB}$ be a class-level lower bound of \(\mathcal{L}\). Let \(\zeta = (\zeta_1, \dots, \zeta_K) \in (\mathbb{S}^{h-1})^K\) be a $K$-point center configuration, and let \(Z^*=(z^*_1, \dots, z^*_B)\in (\mathbb{S}^{h-1})^B\) be a $B$-point configuration which collapses to \(\zeta\), i.e., \(z^*_i = \zeta_{y_i}\) for all \(i \in [B]\). Then, with  $K<h$,
    \begin{align}
        \zeta \in \argmin_{\bar{Z}\in (\mathbb{S}^{h-1})^K} \mathcal{LB}(\bar{Z};K) \iff Z^* \in \argmin_{Z \in (\mathbb{S}^{h-1})^B} \mathcal{L}(Z;Y, \mathcal{B}).
    \end{align}
\end{applem}
\begin{proof}
    For brevity, we omit the dependence on \(Y\) and $\mathcal{B}$ and write \(\mathcal{L}(Z)\) in place of \(\mathcal{L}(Z;Y, \mathcal{B})\).
    We use the following three properties of the class-level lower bound. 
    
    First, for every feature configuration \(Z\in(\mathbb S^{h-1})^B\), the induced class-level configuration 
    \(C(Z)= \bigl(\tfrac{1}{B_1}\sum_{i \in \mathcal{B}_1}z_i, \dots, \tfrac{1}{B_K}\sum_{i \in \mathcal{B}_K} z_i\bigr) \in (\mathbb{B}^{h})^K\) (where $\mathbb{B}$ is the unit ball) satisfies
    \begin{align}
        \mathcal L(Z) \ge \mathcal{LB}(C(Z)).
    \end{align}
    Second, if \(Z\) collapses to a class-level configuration \(\bar Z \in (\mathbb{S}^{h-1})^K\) (i.e., \(z_i = \bar z_{y_i}\) for all \(i\)), then equality holds in the lower bound, that is,
    \begin{align}
        \mathcal L(Z)=\mathcal{LB}(\bar Z).
    \end{align}
    Third, the infimum of $\mathcal{LB}$ over the unit ball is attained on the sphere:
    \begin{align}
        \inf_{\bar Z \in (\mathbb{B}^{h})^K} \mathcal{LB}(\bar Z) \;=\; \inf_{\bar Z \in (\mathbb{S}^{h-1})^K} \mathcal{LB}(\bar Z).
        \label{eq:sphere-reduction}
    \end{align}
    We justify Eq.~\eqref{eq:sphere-reduction} as follows. Observe that each 
    class-level lower bound $\mathcal{LB}$ considered in this paper (Eq.~\eqref{uni-F}) 
    depends on \(\bar Z\) only through the off-diagonal inner products 
    \(\{\bar z_y \cdot \bar z_k\}_{y \neq k}\); the diagonal entries 
    \(\|\bar z_y\|^2\) do not appear in $\mathcal{LB}$.
    
    Fix any \(\bar Z \in (\mathbb{B}^h)^K\) and let \(G \in \mathbb{R}^{K\times K}\) 
    denote its Gram matrix, \(G_{yk} = \bar z_y \cdot \bar z_k\). 
    Then \(G\) is positive semi-definite (PSD) and 
    \(G_{yy} = \|\bar z_y\|^2 \le 1\) for every \(y\). Define
    \begin{align}
        \tilde G \;\coloneqq\; G + \mathrm{diag}\big(1-G_{11},\, \dots,\, 1-G_{KK}\big).
    \end{align}
    Since \(1-G_{yy}\ge 0\), the added diagonal matrix is PSD, so 
    \(\tilde G\) is a sum of two PSD matrices and hence PSD. By construction,
    \begin{align}
        \tilde G_{yy} = 1 \quad \text{for all } y, 
        \qquad 
        \tilde G_{yk} = G_{yk} \quad \text{for all } y \neq k.
    \end{align}
    Because \(\tilde G\) is a \(K\times K\) PSD matrix, \(\mathrm{rank}(\tilde G)\le K\). Since $K<h$, the Cholesky decomposition of \(\tilde G\) 
    yields vectors \(\tilde z_1,\dots,\tilde z_K\in\mathbb{R}^h\) with \(\tilde z_y \cdot \tilde z_k = \tilde G_{yk}\). The diagonal condition \(\tilde G_{yy}=1\) gives \(\|\tilde z_y\|=1\), so \(\tilde Z\coloneqq(\tilde z_1,\dots,\tilde z_K)\in (\mathbb S^{h-1})^K\). 
    Moreover, the off-diagonal inner products of \(\tilde Z\) coincide with those of \(\bar Z\), and since $\mathcal{LB}$ depends only on these,
    \begin{align}
        \mathcal{LB}(\tilde Z) \;=\; \mathcal{LB}(\bar Z).
    \end{align}
    This shows that for every \(\bar Z \in (\mathbb{B}^h)^K\), there exists \(\tilde Z \in (\mathbb S^{h-1})^K\) with the same value of $\mathcal{LB}$. 
    Consequently,
    \begin{align}
        \inf_{\bar Z \in (\mathbb S^{h-1})^K} \mathcal{LB}(\bar Z) 
        \;\le\; \inf_{\bar Z \in (\mathbb{B}^h)^K} \mathcal{LB}(\bar Z).
    \end{align}
    The reverse inequality is immediate from \(\mathbb S^{h-1}\subset \mathbb{B}^h\), which proves Eq.~\eqref{eq:sphere-reduction}.

    Now, we use the properties above to prove the Lemma. We first prove the \((\Longleftarrow)\) direction.
    Assume, for contradiction, that \(\zeta \notin \argmin_{\bar Z \in (\mathbb{S}^{h-1})^K} \mathcal{LB}(\bar Z)\). Then there exists \(\tilde{\zeta}=(\tilde{\zeta}_1,\dots,\tilde{\zeta}_K)\in(\mathbb S^{h-1})^K\) such that
    \begin{align}
        \mathcal{LB}(\tilde{\zeta}) < \mathcal{LB}(\zeta).
    \end{align}
    Define the class-collapsed feature configuration \(\tilde Z=(\tilde z_1,\dots,\tilde z_B)\in(\mathbb S^{h-1})^B\) by
    \begin{align}
        \tilde z_i=\tilde\zeta_{y_i}
        \quad \text{for each } i\in[B].
    \end{align}
    Since \(\tilde Z\) collapses to \(\tilde\zeta\) and \(Z^*\) collapses to \(\zeta\), the equality property gives
    \begin{align}
        \mathcal L(\tilde Z)=\mathcal{LB}(\tilde\zeta), \qquad \mathcal L(Z^*)=\mathcal{LB}(\zeta).
    \end{align}
    Therefore,
    \begin{align}
        \mathcal L(\tilde Z) = \mathcal{LB}(\tilde\zeta) < \mathcal{LB}(\zeta) = \mathcal L(Z^*),
    \end{align}
    which contradicts the assumption that \(Z^*\) is a global minimizer of \(\mathcal L\). Hence,
    \begin{align}
        \zeta\in\argmin_{\bar Z \in (\mathbb{S}^{h-1})^K}\mathcal{LB}(\bar Z).
    \end{align}
    
    It remains to prove the \((\Longrightarrow)\) direction.
    Suppose that \(\zeta\in\argmin_{\bar Z \in (\mathbb{S}^{h-1})^K}\mathcal{LB}(\bar Z)\), and let \(Z^*\) be the class-collapsed feature configuration satisfying \(z_i^*=\zeta_{y_i}\) for each \(i\in[B]\).
    By the sphere reduction property~\eqref{eq:sphere-reduction}, \(\zeta\) is also a global minimizer of $\mathcal{LB}$ over the unit ball:
    \begin{align}
        \zeta\in\argmin_{\bar Z \in (\mathbb{B}^{h})^K}\mathcal{LB}(\bar Z).
        \label{eq:zeta-ball-min}
    \end{align}
    Let \(Z\in(\mathbb S^{h-1})^B\) be an arbitrary feature configuration. 
    By the lower bound property,
    \begin{align}
        \mathcal L(Z)\ge \mathcal{LB}(C(Z)).
    \end{align}
    Since \(C(Z)\in (\mathbb{B}^{h})^K\), \eqref{eq:zeta-ball-min} yields
    \begin{align}
        \mathcal{LB}(C(Z))\ge \mathcal{LB}(\zeta).
    \end{align}
    Moreover, because \(Z^*\) collapses to \(\zeta\), the equality property gives
    \begin{align}
        \mathcal L(Z^*)=\mathcal{LB}(\zeta).
    \end{align}
    Combining these relations,
    \begin{align}
        \mathcal L(Z) \;\ge\; \mathcal{LB}(C(Z)) \;\ge\; \mathcal{LB}(\zeta) \;=\; \mathcal L(Z^*).
    \end{align}
    Since this holds for every \(Z\in(\mathbb S^{h-1})^B\), \(Z^*\) is a global minimizer of \(\mathcal L\).
\end{proof}
\subsection{A condition of the lower bound for a regular simplex optimality}\label{Thm-proof-derivation}
Here, we give a detailed version of Lemma~\ref{lem-minimizer_of_F}.
The class-level lower bounds share a common functional form,
\begin{align}\label{uni-F}
    \mathcal{LB}(\bar{Z};K) &= \frac{1}{\beta}\sum_{y \in [K]} a_y \log\left(\Gamma_y(\bar{Z})\right),\\
    \text{where} \quad \Gamma_y(\bar{Z}) &= b_y + e^{-1}\sum_{k \in [K] \backslash \singleton{y}}c_k\exp(\bar{z}_y \cdot \bar{z}_k),
\end{align}
with \(\beta=\sum_{y \in [K]}a_y\). The coefficients \(\{a_y\}, \{b_y\}, \{c_y\}\) arise from applying Jensen's inequality during the derivation of the lower bound, and \(\beta\) is a normalization factor. The following Lemma states a single condition under which a regular simplex satisfies the first-order Lagrangian condition for any objective of the form in Eq.~\eqref{uni-F}.

\begin{applem}\label{appendix-lem:unified-lagrangian}(Necessary condition for regular simplex optimality)
Let \(\bar{Z}^*\) be a \(K\)-point configuration forming a regular simplex on \(\mathbb{S}^{h-1}\). \(\bar{Z}^*\) satisfies the first-order condition
\begin{align}
    \nabla_{\bar{z}_i} \mathcal{LB}(\bar{Z}^*) = \lambda_i \nabla_{\bar{z}_i} \left(\| \bar{z}_i \|^2 -1\right) \quad \text{for all } i\in [K],
\end{align}
if and only if, for every \(i \in [K]\),
\begin{align}
    \frac{a_i c_k}{\Gamma_i} + \frac{a_k c_i}{\Gamma_k}
    = \text{Constant}, \quad \text{for all } k \in [K]\backslash \singleton{i},
    \label{Lagrangian-uni-coefficient}
\end{align}
where \(\Gamma_y = b_y + e^{d-1}\sum_{k \in [K] \backslash \singleton{y}} c_k\) and \(d=\bar z_i^* \cdot \bar z_j^*\) for all \(i\neq j\).
\end{applem}

When such \(\lambda_i\) exists, it is given by
\begin{align}\label{Lagrangian-uni-lambda}
    \lambda_i = -\frac{e^{d-1}}{2\beta}\left( \frac{a_ic_k}{\Gamma_i} + \frac{a_k c_i}{\Gamma_k}\right).
\end{align}
\begin{proof}
Denote \(\Gamma_y(\bar{Z})=\Gamma_y\) for simplicity. Consider the constrained optimization problem of minimizing \(\mathcal{LB}(\bar{Z};K)\) subject to \(G_i(\bar{Z})=\|\bar{z}_i\|^2-1=0\) for all \(i\in[K]\). Then \(\bar{Z}^*\) satisfies the first-order condition of the associated Lagrangian if and only if for each \( i \in [K]\) there exists \(\lambda_i \in \mathbb{R}\) such that
\begin{align}\label{uni-Lagrangian}
    \nabla_{\bar{z}_i}\mathcal{LB}(\bar{Z}^*)=\lambda_i \nabla_{\bar{z}_i}G_i(\bar{Z}^*).
\end{align}
Fix an arbitrary class \(i\in[K]\). We first derive the left-hand side of Eq.~\eqref{uni-Lagrangian}. The gradient of \(\Gamma_y\) is
\begin{subequations}
    \begin{empheq}[left={\nabla_{\bar{z}_i} \Gamma_y = \empheqlbrace\,}]{align}
        & e^{-1} \sum_{k \in [K]\backslash \{y\}} c_k \exp(\bar{z}_y \cdot \bar{z}_k) \bar{z}_k, && \text{if } y=i, \label{uni-eq-grad-y} \\
        & e^{-1} c_i \exp(\bar{z}_y \cdot \bar{z}_i)\bar{z}_y, && \text{if } y\in [K]\backslash \singleton{i}.
        \label{uni-eq-grad-k}
    \end{empheq}
\end{subequations}
Hence,
{\small
\begin{align}
    \nabla_{\bar{z}_i} \mathcal{LB}(\bar{Z};K)
    &= \nabla_{\bar{z}_i} \left( \frac{1}{\beta} \sum_{y \in [K]} a_y \log (\Gamma_y) \right) \\
    &= \frac{1}{\beta} \left(a_i \nabla_{\bar{z}_i}\log(\Gamma_i) + \sum_{y \in [K] \backslash \singleton{i}}a_y\nabla_{\bar{z}_i} \log(\Gamma_y) \right) \\
    &\overset{\text{Eq. }~\eqref{uni-eq-grad-y},~\eqref{uni-eq-grad-k}}{=} \!\!\frac{1}{\beta}\left(\! a_i \frac{e^{-1}\sum_{k \in[K]\backslash\singleton{i}}c_k\exp(\bar{z}_i\cdot \bar{z}_k)\bar{z}_k}{\Gamma_i} \! + \!\!\!\! \sum_{y \in [K]\backslash \singleton{i}}\!\!\!a_y \frac{e^{-1} c_i\exp(\bar{z}_y \cdot \bar{z}_i)\bar{z}_y}{\Gamma_y}\right) \\
    &= \frac{e^{-1}}{\beta} \sum_{k\in[K]\backslash\singleton{i}} \left( \frac{a_i c_k}{\Gamma_i} + \frac{a_k c_i}{\Gamma_k}\right) \exp(\bar{z}_i \cdot \bar{z}_k) \bar{z}_k.
\end{align}}
Next, we derive the right-hand side:
\begin{align}
    \lambda_i \nabla_{\bar{z}_i} G_i(\bar{Z})
    &= \lambda_i \nabla_{\bar{z}_i} \left( \|\bar{z}_i \|^2 -1 \right)
    = 2\lambda_i \bar{z}_i.
\end{align}
Therefore, the Lagrangian equation~\eqref{uni-Lagrangian} becomes
\begin{align}\label{Lagrangian-uni-derived}
     \frac{e^{-1}}{\beta} \sum_{k\in[K]\backslash\singleton{i}} \left( \frac{a_i c_k}{\Gamma_i} + \frac{a_k c_i}{\Gamma_k}\right) \exp(\bar{z}_i \cdot \bar{z}_k) \bar{z}_k =  2\lambda_i \bar{z}_i.
\end{align}
We now evaluate this condition at \(\bar{Z}=\bar{Z}^*\) and denote \(\Gamma_y(\bar{Z}^*)=\Gamma_y\). Since \(\bar{Z}^*\) is a regular simplex, there exists \(d\in\mathbb{R}\) such that \(\bar z_i^*\cdot \bar z_j^*=d\) for all \(i\neq j\). Thus,
\begin{align}
    &\frac{e^{d-1}}{\beta} \sum_{k\in[K]\backslash\singleton{i}} \left( \frac{a_i c_k}{\Gamma_i} + \frac{a_k c_i}{\Gamma_k}\right) \bar{z}_k^* =  2\lambda_i \bar{z}_i^*, \\
    &\text{where} \quad \Gamma_y = b_y + e^{d-1}\sum_{k \in [K] \backslash \singleton{y}} c_k.
\end{align}
By the regular simplex property \(\bar{z}_i^* = -\sum_{k\in[K]\backslash\singleton{i}} \bar{z}_k^*\), we can merge the right-hand side into the left-hand side and obtain
\begin{align}
    \sum_{k \in [K]\backslash \singleton{i}}\left( \frac{a_i c_k}{\Gamma_i} + \frac{a_k c_i}{\Gamma_k} + \frac{2\beta}{e^{d-1}}\lambda_i \right) \bar{z}_k^* = 0.
    \label{Lagrangian-uni-regsim}
\end{align}
Since any subset of \(K-1\) vertices from a regular simplex is linearly independent, the coefficients in Eq.~\eqref{Lagrangian-uni-regsim} must all vanish. Such a \(\lambda_i\) exists if and only if
\begin{align}
    \frac{a_i c_k}{\Gamma_i} + \frac{a_k c_i}{\Gamma_k}
    = \text{Constant}, \quad \text{for all } k \in [K]\backslash \singleton{i},
\end{align}
which proves the result. The expression for \(\lambda_i\) in Eq.~\eqref{Lagrangian-uni-lambda} follows immediately.
\end{proof}

\subsection{Applying our framework to supervised contrastive objectives}\label{Thm-proof-verification}

We now apply the framework to supervised contrastive objectives in the theorems. For each loss, the procedure is the same: 1. Derive the lower bound using Lemmas~\ref{appendix-lem-anchor-wise} and~\ref{appendix-lem-class-wise}. 2. From the lower bound, we identify the coefficients \(\{a_y\}, \{b_y\}, \{c_y\}, \beta\); substitute them into the condition of Lemma~\ref{appendix-lem:unified-lagrangian}; and check whether the resulting equality holds across classes. 3. Transfer the result back to the original loss by Lemma~\ref{appendix-lem-F-L}.

For SCL and BCL, dependence on the class cardinalities \(\{B_y\}\) prevents the condition of Lemma~\ref{appendix-lem:unified-lagrangian} from being satisfied, so the regular simplex fails to be even a stationary point of the lower bound and therefore fails to be a global minimizer. For FBCL, the condition holds since structural balancing eliminates dependence on \(\{B_y\}\). We further show that this stationary point is indeed a global minimum by deriving subsequent lower bounds whose equality conditions characterize a regular simplex.

\subsubsection{SCL does not yield a regular simplex at the optimum under long-tailed distributions }\label{Appendix:Thm-proof-scl}

Theorem~\ref{Appendix:Thm-SCL} states that under long-tailed distributions, the SCL loss does not attain a regular simplex as the geometry of its global minimum. We prove this by showing that the regular simplex configuration \(\bar{Z}^*\) does not satisfy the first-order Lagrangian condition for \(\mathcal{LB}_{\mathrm{SCL}}\), and applying Lemma~\ref{lem-F-L}.

\begin{appthm}\label{Appendix:Thm-SCL}
    Let \(\bar{Z}^* =(\bar{z}_1^*, \dots, \bar{z}_K^*)\) be a $K$-point configuration forming a regular simplex on \(\mathbb{S}^{h-1}\), and \(Z^* = (z^*_1, \dots, z^*_B)\) be a configuration which collapses to \(\bar{Z}^*\). Consider the constrained optimization problem of minimizing \(\mathcal{LB}_{\mathrm{SCL}}\) defined in Proposition~\ref{prop-scl} subject to \( G_i(\bar{Z})=\|\bar{z}_i\|^2 -1 =0\) for all \(i \in[K]\). Then,
    \begin{align}
        \exists\, i \in [K] \text{ such that } \nexists\, \lambda_i \in \mathbb{R} \text{ for which } \nabla_{\bar{z}_i}\mathcal{LB}_{\mathrm{SCL}}(\bar{Z}^*)=\lambda_i \nabla_{\bar{z}_i}G_i(\bar{Z}^*).
    \end{align}
    Consequently, \(\bar{Z}^*\) fails to be a stationary point of \(\mathcal{LB}_{\mathrm{SCL}}\), and therefore fails to be a global minimizer. By Lemma~\ref{lem-F-L},
    \begin{align}
        Z^* \notin \argmin_{Z \in (\mathbb{S}^{h-1})^B} \mathcal{L}_{\mathrm{SCL}}(Z;Y, \mathcal{B}).
    \end{align}
\end{appthm}
\paragraph{Proof of Theorem~\ref{method:Thm-scl}.}
\begin{proof}
By Proposition~\ref{prop-scl}, \(\mathcal{L}_{\mathrm{SCL}}\) is bounded below by
\begin{align}\label{Thm1-LB-scl}
    \mathcal{L}_{\mathrm{SCL}}(Z;Y, \mathcal{B}) &\overset{\text{Prop. }\ref{prop-scl}}{\ge} \mathcal{LB}_{\mathrm{SCL}}(\bar{Z};K) = \frac{1}{B} \sum_{y \in [K]} B_y \log \left( \Gamma_y \right),\\
    \text{where} \quad
    \Gamma_y &= B_y-1 + e^{-1}\sum\limits_{k \in [K]\backslash \singleton{y}} B_k \exp(\bar{z}_y \cdot \bar{z}_k).
\end{align}
Substituting \(a_y=B_y\), \(b_y=B_y-1\), \(c_y=B_y\), and \(\beta=B\) into Eq.~\eqref{uni-F} yields \(\mathcal{LB}_{\mathrm{SCL}}(\bar{Z};K)\). Thus the condition Eq.~\eqref{Lagrangian-uni-coefficient} in Lemma~\ref{appendix-lem:unified-lagrangian} becomes
\begin{align}
    \frac{B_i B_k}{\Gamma_i} + \frac{B_k B_i}{\Gamma_k} &= \text{Constant}, \quad \text{for all } k \in [K]\backslash \singleton{i}, \label{Lagrangian-scl-coefficient}\\
    \text{where} \quad \Gamma_y &= B_y-1 + e^{d-1} \sum_{k \in [K] \backslash \singleton{y}}B_k. \notag
\end{align}
Since \(B_i\) remains fixed in \eqref{Lagrangian-scl-coefficient}, it suffices to examine the dependence on $k$, determined by
\begin{align}
    \frac{B_k}{\Gamma_i} + \frac{B_k}{\Gamma_k} &= \text{Constant}, \quad \text{for all } k \in [K]\backslash \singleton{i}.
\end{align}
Given \(K\ge3\), there exist distinct indices \(i\neq r,s\in[K]\) such that \(B_r > B_s\). We prove 
\[
\frac{B_r}{\Gamma_i} + \frac{B_r}{\Gamma_r} \neq \frac{B_s}{\Gamma_i} + \frac{B_s}{\Gamma_s}
\]
by showing that \(\frac{B_y}{\Gamma_y}\) is a monotonically increasing function of \(B_y\). Indeed,
\begin{align}
    \Gamma_y &= B_y -1 + e^{d-1}\sum_{k \in [K] \backslash \singleton{y}}B_k\\
    &= B_y -1 + (B-B_y)e^{d-1}\\
    &= (1-e^{d-1})B_y + (Be^{d-1}-1).
\end{align}
Note that \(d = \bar{z}_i^* \cdot \bar{z}_j^* = -1/(K-1)\) for any \(i \neq j \in [K]\) by condition~\ref{def:regular_simplex_2}. Hence \(1-e^{d-1}>0\). Also, given \(B>4\) and \(K\ge3\), we have \(Be^{d-1}-1>0\). Therefore, \(\frac{B_y}{\Gamma_y}\) is increasing in \(B_y\), and since $\Gamma_i$ also remains fixed, with $B_r > B_s$,
\[
\frac{B_r}{\Gamma_i} + \frac{B_r}{\Gamma_r} > \frac{B_s}{\Gamma_i} + \frac{B_s}{\Gamma_s}.
\]
Consequently, condition \eqref{Lagrangian-scl-coefficient} cannot be satisfied. By Lemma~\ref{appendix-lem:unified-lagrangian}, \(\bar{Z}^*\) does not satisfy the Lagrangian first-order condition for \(\mathcal{LB}_{\mathrm{SCL}}\), i.e.,
\begin{align}
    \bar{Z}^* \notin \argmin_{\bar{Z} \in (\mathbb{S}^{h-1})^K} \mathcal{LB}_{\mathrm{SCL}}(\bar{Z};K).
\end{align}
By Lemma~\ref{lem-F-L},
\begin{align}
    Z^* \notin \argmin_{Z \in (\mathbb{S}^{h-1})^B} \mathcal{L}_{\mathrm{SCL}}(Z;Y, \mathcal{B}),
\end{align}
i.e., the SCL loss fails to form a regular simplex when attaining its minimum.
\end{proof}
\paragraph{Remark.} \citet{yi2025geometry} and \citet{kini2024symmetric} analyzed the representation geometry of SCL
under class imbalance assuming a controlled distribution or the non-negativity of features, but they did not show that SCL fails to form a regular simplex without such assumptions, as Theorem~\ref{Appendix:Thm-SCL} does. Moreover, neither work proposes an objective that provably induces a regular simplex under class imbalance.
\subsubsection{BCL does not yield a regular simplex at the optimum under long-tailed distributions }\label{Appendix:Thm-proof-bcl}
Theorem~\ref{Appendix:Thm-BCL} states that under any long-tailed distribution, the BCL loss does not yield a regular simplex as the geometry of its global minimum. 
\paragraph{Why BCL failed to induce a regular simplex.}
The original work of BCL derives a final lower bound that induces regular simplex conditions. However, our framework reveals that the final lower bound is not tight; this is because the application of Jensen's inequality to $\sum_{i \in \mathcal{B}}$ is not tight (it should be used sequentially to $\sum_{i \in \mathcal{B}_y}$ and then to $\sum_{y \in [K]}$). We show this by showing that the regular simplex configuration already fails to minimize the class-level lower bound, shown below.
\begin{appthm} \label{Appendix:Thm-BCL}
    Let \(\bar{Z}^*, Z^*\) be defined as in Theorem~\ref{Appendix:Thm-SCL}. Consider the constrained optimization problem of minimizing \(\mathcal{LB}_{\mathrm{BCL}}\) defined in Proposition~\ref{prop-bcl} subject to \( G_i(\bar{Z})=\|\bar{z}_i\|^2 -1 =0\) for all \(i \in[K]\). Then,
    \begin{align}
        \exists\, i \in [K] \text{ such that } \nexists\, \lambda_i \in \mathbb{R} \text{ for which } \nabla_{\bar{z}_i}\mathcal{LB}_{\mathrm{BCL}}(\bar{Z}^*)=\lambda_i \nabla_{\bar{z}_i}G_i(\bar{Z}^*).
    \end{align}
    Consequently, \(\bar{Z}^*\) fails to be a stationary point of \(\mathcal{LB}_{\mathrm{BCL}}\), and therefore fails to be a global minimizer. By Lemma~\ref{lem-F-L},
    \begin{align}
        Z^* \notin \argmin_{Z \in (\mathbb{S}^{h-1})^B} \mathcal{L}_{\mathrm{BCL}}(Z;Y, \mathcal{B}).
    \end{align}
\end{appthm}
\paragraph{Proof of Theorem~\ref{method:Thm-bcl}.}
\begin{proof}
By Proposition~\ref{prop-bcl}, \(\mathcal{L}_{\mathrm{BCL}}\) is bounded below by
\begin{align}
    \mathcal{L}_{\mathrm{BCL}}(Z;Y, \mathcal{B}) &\overset{\text{Prop. }\ref{prop-bcl}}{\ge} \mathcal{LB}_{\mathrm{BCL}}(\bar{Z};K) = \frac{1}{B} \sum_{y \in [K]} B_y \log \left( \Gamma_y \right),\\
    \text{where} \quad
    \Gamma_y &= 1 + e^{-1}\sum\limits_{k \in [K]\backslash \singleton{y}} \exp(\bar{z}_y \cdot \bar{z}_k).
\end{align}
Substituting \(a_y=B_y\), \(b_y=c_y=1\), and \(\beta=B\) into Eq.~\eqref{uni-F} yields \(\mathcal{LB}_{\mathrm{BCL}}(\bar{Z};K)\). Thus the condition Eq.~\eqref{Lagrangian-uni-coefficient} in Lemma~\ref{appendix-lem:unified-lagrangian} becomes
\begin{align}
    \frac{B_i}{\Gamma_i} + \frac{B_k}{\Gamma_k} &= \text{Constant}, \quad \text{for all } k \in [K]\backslash \singleton{i}, \label{Lagrangian-bcl-coefficient}\\
    \text{where} \quad \Gamma_y &= 1 + (K-1)e^{d-1}. \notag
\end{align}
It follows that \(\Gamma_1=\Gamma_2=\cdots=\Gamma_K\). Given \(K\ge3\), there exist distinct indices \(r,s\in[K]\) such that \(B_r > B_s\). Then,
\[
\frac{B_i}{\Gamma_i} + \frac{B_r}{\Gamma_r}
>
\frac{B_i}{\Gamma_i} + \frac{B_s}{\Gamma_s},
\]
so condition Eq.~\eqref{Lagrangian-bcl-coefficient} cannot be satisfied. By Lemma~\ref{appendix-lem:unified-lagrangian}, \(\bar{Z}^*\) does not satisfy the Lagrangian first-order condition for \(\mathcal{LB}_{\mathrm{BCL}}\), and thus
\begin{align}
    \bar{Z}^* \notin \argmin_{\bar{Z} \in (\mathbb{S}^{h-1})^K} \mathcal{LB}_{\mathrm{BCL}} (\bar{Z}).
\end{align}
By Lemma~\ref{lem-F-L},
\begin{align}
    Z^* \notin \argmin_{Z \in (\mathbb{S}^{h-1})^B} \mathcal{L}_{\mathrm{BCL}}(Z;Y, \mathcal{B}),
\end{align}
i.e., the BCL loss fails to form a regular simplex when attaining its minimum.
\end{proof}
\paragraph{Optimality of SCL and BCL.} Characterizing the optimality of SCL and BCL under long-tailed distribution is a harder problem which has to deal with numerical relationships across angles between class centers \citep{yi2025geometry}. Therefore, we rather investigate whether a particular geometry is attained at the minimum of an objective.
\subsubsection{FBCL yields a regular simplex at the optimum under long-tailed distributions}\label{Appendix:Thm-proof-fbcl}
\paragraph{Balancing the aggregation across anchors.}
Our analysis reveals that balancing within each anchor-wise loss is not sufficient to remove the dependence of the objective on class cardinality. In SCL, the positive samples are averaged within each anchor-wise loss, and BCL additionally balances the contributions of negative classes in its denominator. Moreover, \citet{zhong2022rebalanced}, \citet{de2025long}, and \citet{roh2026equilibrium} have studied re-balancing supervised contrastive objectives using the effective number of samples \citep{cui2019class} or class cardinalities in the numerator or denominator; \textbf{these methods still perform balancing within each anchor-wise loss}.
However, when these anchor-wise losses are aggregated to form the overall objective, each class contributes in proportion to its number of anchors. This remaining class-dependent weighting is consequently inherited by the class-level lower bound, which prevents the regular simplex from being optimal under long-tailed distributions. Structural balancing (SB) normalizes the aggregation of anchor-wise losses across classes. Although this operation can be viewed as a class-dependent weighting of anchor losses, its role here is determined by the geometric analysis: SB removes the remaining class-cardinality dependence from the class-level lower bound and thereby recovers the regular simplex as an optimum. FBCL combines this structural balancing with the positive and negative balancing already present in BCL. 

Theorem~\ref{Appendix:Thm-FBCL} states that the regular simplex is a global minimizer of the FBCL loss. The proof proceeds in two steps. We first show that the regular simplex configuration \(\bar{Z}^*\) satisfies the first-order Lagrangian condition for \(\mathcal{LB}_{\mathrm{FBCL}}\), and is therefore a stationary point. We then show that \(\bar{Z}^*\) is in fact a global minimizer of \(\mathcal{LB}_{\mathrm{FBCL}}\) by further deriving a lower bound, with equality attained at the regular simplex.
\begin{appthm}\label{Appendix:Thm-FBCL}
    Let \(\bar{Z}^*, Z^*\) be defined as in Theorem~\ref{Appendix:Thm-SCL}. Consider the constrained optimization problem of minimizing \(\mathcal{LB}_{\mathrm{FBCL}}\) defined in Proposition~\ref{prop-FBCL} subject to \( G_i(\bar{Z})=\|\bar{z}_i\|^2 -1 =0\) for all \(i \in[K]\). We show that, for each \(i \in [K]\), there exists \(\lambda_i \in \mathbb{R}\) such that
    \begin{align}
        \nabla_{\bar{z}_i}\mathcal{LB}_{\mathrm{FBCL}}(\bar{Z}^*)=\lambda_i \nabla_{\bar{z}_i}G_i(\bar{Z}^*),
    \end{align}
    where \(d = \bar{z}^*_i \cdot \bar{z}^*_j\) for any \(i \neq j\). Moreover, we show that \(\bar{Z}^*\) is a global minimizer of \(\mathcal{LB}_{\mathrm{FBCL}}\):
    \begin{align}
        Z^* \in \argmin_{Z \in (\mathbb{S}^{h-1})^B} \mathcal{L}_{\mathrm{FBCL}}(Z;Y, \mathcal{B}).
    \end{align}
    Specifically, the loss is bounded by the following minimum:
    \begin{align}
        \mathcal{L}_{\mathrm{FBCL}}(Z;Y, \mathcal{B}) \ge \log \left( 1 + (K-1)\exp\left(-\frac{1}{K-1} -1\right) \right),
    \end{align}
    where equality is attained if and only if \(Z\) collapses to \(Z^*\), a regular simplex.
\end{appthm}
\paragraph{Proof of Theorem~\ref{method:Thm-fbcl}.}
\begin{proof}
The proof proceeds in two stages. We first show that the regular simplex \(\bar{Z}^*\) is a stationary point of \(\mathcal{LB}_{\mathrm{FBCL}}\) via the unified Lagrangian condition. We then extend this to global optimality by deriving the lower bound of \(\mathcal{LB}_{\mathrm{FBCL}}\).

\textit{Stage 1: Stationary point.} By Proposition~\ref{prop-FBCL}, \(\mathcal{L}_{\mathrm{FBCL}}\) is bounded below by
\begin{align}
    \mathcal{L}_{\mathrm{FBCL}}(Z;Y, \mathcal{B}) &\overset{\text{Prop. }\ref{prop-FBCL}}{\ge} \mathcal{LB}_{\mathrm{FBCL}}(\bar{Z};K)=\frac{1}{K} \sum_{y \in [K]} \log \left( \Gamma_y \right),\\
    \text{where} \quad
    \Gamma_y &= 1 + e^{-1}\sum\limits_{k \in [K]\backslash \singleton{y}} \exp(\bar{z}_y \cdot \bar{z}_k).
\end{align}
Substituting \(a_y=b_y=c_y=1\) and \(\beta=K\) into Eq.~\eqref{uni-F} yields \(\mathcal{LB}_{\mathrm{FBCL}}(\bar{Z};K)\). Thus the condition Eq.~\eqref{Lagrangian-uni-coefficient} in Lemma~\ref{appendix-lem:unified-lagrangian} becomes
\begin{align}
    \frac{1}{\Gamma_i} + \frac{1}{\Gamma_k} &= \text{Constant}, \quad \text{for all } k \in [K]\backslash \singleton{i}, \label{Lagrangian-FBCL-coefficient}\\
    \text{where} \quad \Gamma_y &= 1 + (K-1)e^{d-1}. \notag
\end{align}
Since \(\Gamma_1=\Gamma_2=\cdots=\Gamma_K\), condition Eq.~\eqref{Lagrangian-FBCL-coefficient} is automatically satisfied. The corresponding multiplier follows from Eq.~\eqref{Lagrangian-uni-lambda}:
\begin{align}
    \lambda_i \overset{\text{Eq.}~\eqref{Lagrangian-uni-lambda}}{=} -\frac{e^{d-1}}{K + K(K-1)e^{d-1}}.
\end{align}
By Lemma~\ref{appendix-lem:unified-lagrangian}, the regular simplex configuration \(\bar{Z}^*\) satisfies the Lagrangian condition, and thus \(\bar{Z}^*\) serves as a stationary point of \(\mathcal{LB}_{\mathrm{FBCL}}\).

\textit{Stage 2: Global optimality.} We further derive the lower bound of \(\mathcal{LB}_{\mathrm{FBCL}}\) and eventually show that the regular simplex configuration \(\bar{Z}^*\) is a global minimizer of \(\mathcal{LB}_{\mathrm{FBCL}}\):
\begin{align}
    \mathcal{LB}_{\mathrm{FBCL}}(\bar{Z};K) &= \frac{1}{K}\sum_{y \in [K]} \log \left( 1+ e^{-1}\sum_{k \in [K]\backslash \singleton{y}}\exp(\bar{z}_y \cdot \bar{z}_k) \right) \label{proof-FBCL-lower-bound} \\
     &\ge \frac{1}{K}\sum_{y \in [K]} \log \left( 1+ e^{-1}(K-1)\exp\left(\frac{1}{K-1}\sum_{k \in [K]\backslash \singleton{y}}\bar{z}_y \cdot \bar{z}_k\right) \right) \label{proof-FBCL-lower-bound-2-regsim}\\
     &= \frac{1}{K}\sum_{y \in [K]} \phi\left( \frac{1}{K-1} \sum_{k \in [K] \backslash \singleton{y}} \bar{z}_y \cdot \bar{z}_k \right) \label{proof-FBCL-lower-bound-3-phi}\\
     &\ge \phi \left( \frac{1}{K}\sum_{y \in [K]}\frac{1}{K-1} \sum_{k \in [K] \backslash \singleton{y}} \bar{z}_y \cdot \bar{z}_k \right)\label{proof-FBCL-lower-bound-4-phi-Jensen} \\
     &= \phi \left( \frac{1}{K(K-1)} \sum_{\substack{i,j\in [K] \\ i \neq j}} \bar{z}_i \cdot \bar{z}_j \right) \label{proof-FBCL-lower-bound-5-same}\\
     &\ge \phi \left( -\frac{1}{K-1} \right)\label{proof-FBCL-lower-bound-6-final}\\
     &= \log \left( 1 + (K-1)\exp\left(-\frac{1}{K-1} -1\right)\right).
\end{align}
We justify each step in turn. The transition \eqref{proof-FBCL-lower-bound} \(\to\) \eqref{proof-FBCL-lower-bound-2-regsim} follows from Jensen's inequality: since \(\exp(x)\) is convex and \(\log(x)\) is monotonically increasing, the inner average over \(k\) can be moved inside the exponential. Defining \(\phi(x) = \log(1 + (K-1)\exp(x-1))\) gives \eqref{proof-FBCL-lower-bound-3-phi}. Since \(\phi\) is itself a convex function, Jensen's inequality applied to the outer average over \(y\) yields \eqref{proof-FBCL-lower-bound-3-phi} \(\to\) \eqref{proof-FBCL-lower-bound-4-phi-Jensen}. The equality \eqref{proof-FBCL-lower-bound-4-phi-Jensen} \(\to\) \eqref{proof-FBCL-lower-bound-5-same} simply rewrites the double sum. Finally, the transition \eqref{proof-FBCL-lower-bound-5-same} \(\to\) \eqref{proof-FBCL-lower-bound-6-final} follows from the inequality
\begin{align}
    \sum_{i \ne j} \bar{z}_i \cdot \bar{z}_j  = 
    \left\| \sum_{i=1}^K \bar{z}_i \right\|^2 - \sum_{i=1}^K \| \bar{z}_i \|^2 \ge -K,
\end{align}
which gives \(\sum_{i \ne j} \bar{z}_i \cdot \bar{z}_j \ge -K\), and hence \(\frac{1}{K(K-1)}\sum_{i \ne j} \bar{z}_i \cdot \bar{z}_j \ge -\frac{1}{K-1}\). Combined with the monotonicity of \(\phi\), this yields \eqref{proof-FBCL-lower-bound-6-final}.

Equality is attained if and only if both
\begin{enumerate}
    \item[(1)] for all \(i \neq j \in [K]\), \(\bar{z}_i \cdot \bar{z}_j = -\frac{1}{K-1}\), and
    \item[(2)] \(\sum_{i=1}^K \bar{z}_i = 0\)
\end{enumerate}
hold. These conditions are identical to the regular simplex conditions of Definition~\ref{sec:prelim_def_regular}
Therefore, attaining the minimum of \(\mathcal{LB}_{\mathrm{FBCL}}\), and consequently of $\mathcal{L}$, yields a regular simplex. 
\begin{align}
    \log \left( 1 + (K-1)\exp\left(-\tfrac{1}{K-1} -1\right) \right) =\log \left( 1 + (K-1)\exp\big(-\tfrac{K}{K-1} \big) \right).
\end{align}
\end{proof}

\subsection{Lower bounds of supervised contrastive objectives}\label{Prop-proof}
The class-level lower bounds \(\mathcal{LB}_{\mathrm{SCL}}\), \(\mathcal{LB}_{\mathrm{BCL}}\), and \(\mathcal{LB}_{\mathrm{FBCL}}\) (Propositions~\ref{appendix-prop-scl},~\ref{appendix-prop-bcl}, and~\ref{appendix-prop-FBCL}) are derived through a common two-step procedure: an anchor-wise lower bound followed by a class-wise lower bound. We state these two steps as Lemmas~\ref{appendix-lem-anchor-wise} and~\ref{appendix-lem-class-wise}, and then apply them sequentially to obtain each objective. SCL and BCL differ in their negative class weights, while FBCL additionally modifies the aggregation of anchor-wise losses. We unify this difference by introducing a coefficient \(w_k\) on each negative class, where \(w_k = B_k\) for SCL and \(w_k = 1\) for BCL and FBCL.
\subsubsection{Lower bound of the anchor-wise loss (Lemma~\ref{appendix-lem-anchor-wise})}
Lemma~\ref{appendix-lem-anchor-wise} provides a lower bound for the anchor-wise loss of a single sample \(z_i\).
\begin{applem}\label{appendix-lem-anchor-wise}(Lower bound of the anchor-wise loss)
    Let \(Z, \bar{Z}, Y\) be defined as in Proposition~\ref{appendix-prop-scl}. For any \(Z\) and \(Y\), fix an anchor \(z_i\) with class label \(y\). Then, the anchor-wise loss is bounded below by:
    \begin{align}
          -\frac{1}{B_y-1}\!\!\!\sum_{j\in \mathcal{B}_y \backslash\singleton{i}}\!\!\log \!\left(\frac{\exp(z_i\! \cdot\! z_j)}{\sum_{k \in [K]} \sum_{m \in \mathcal{B}_k}\!\frac{w_k}{B_k}\exp(z_i \!\cdot \!z_m)}\!\right) \ge
          \log\!\left(\!w_y + e^{-1}\!\!\!\!\!\!\sum_{k \in [K]\backslash \singleton{y}}\!\!\! w_k\exp(z_i \!\cdot\! \bar{z}_k)\!\right),
    \end{align}
    where when \(k=y\), the anchor $i$ is excluded from the summation in the denominator, and $B_k$ is replaced by $B_y -1$.
    For the weight $w_k>0$, the case \(w_k = B_k\) (with \(w_y = B_y - 1\) when \(k = y\)) corresponds to SCL, and \(w_k = 1\) for all \(k \in [K]\) corresponds to BCL and FBCL. Equality is attained if and only if the following conditions are satisfied:
    \begin{enumerate}[label=(C\arabic*), resume*=LBconditions]
        \item \(z_i = z_j,\) for all \(i, j \in \mathcal{B}_y\). \label{(C2)}
        \item For each \(k \in [K] \backslash \singleton{y}\), there exists \(\delta_{k,i} \in \mathbb{R}\) such that \(z_i \cdot z_m = \delta_{k,i}\), for all \(m \in \mathcal{B}_k\). \label{(C3)}
    \end{enumerate}
\end{applem}
\begin{proof}
The anchor-wise loss can be rewritten as
\begin{align}
    -\frac{1}{B_y-1}&\sum_{j\in \mathcal{B}_y \backslash\singleton{i}}\log \left(\frac{\exp(z_i \cdot z_j)}{\sum_{k \in [K]} \sum_{m \in \mathcal{B}_k}\frac{w_k}{B_k}\exp(z_i \cdot z_m)}\right) \\
    &= \frac{1}{B_y-1}\sum_{j\in \mathcal{B}_y \backslash\singleton{i}}\log \left(\frac{\sum_{k \in [K]} \sum_{m \in \mathcal{B}_k}\frac{w_k}{B_k}\exp(z_i \cdot z_m)}{\exp(z_i \cdot z_j)}\right) \\
    &= \log \left(\frac{\sum_{k \in [K]} \sum_{m \in \mathcal{B}_k}\frac{w_k}{B_k}\exp(z_i \cdot z_m)}{\exp\left(\frac{1}{B_y-1}\sum_{j \in \mathcal{B}_y \backslash\singleton{i}}z_i \cdot z_j\right)}\right)\\
    &= \log \left(\frac{\sum_{p \in \mathcal{B}_y\backslash\singleton{i}}\frac{w_y}{B_y-1}\exp(z_i \cdot z_p) + \sum_{k \in [K]\backslash\singleton{y}} \sum_{m \in \mathcal{B}_k}\frac{w_k}{B_k}\exp(z_i \cdot z_m)}{\exp\left(\frac{1}{B_y-1}\sum_{j \in \mathcal{B}_y \backslash\singleton{i}}z_i \cdot z_j\right)}\right).
\end{align}
Applying Jensen's inequality to the first term in the numerator, we obtain
\begin{align}
    \frac{1}{B_y-1}\sum_{p \in \mathcal{B}_y \backslash \singleton{i}}\exp{(z_i \cdot z_p)} \ge \exp{\left(\frac{1}{B_y-1}\sum_{p \in \mathcal{B}_y \backslash \singleton{i}} z_i \cdot z_p \right)}, \label{lem1-bound-1}
\end{align}
where equality is attained if and only if
\begin{itemize}
    \item[(a)] there is a constant \(\delta_i\) such that \(z_i \cdot z_p = \delta_i, \text{ for all } p \in \mathcal{B}_y \backslash \singleton{i}\).
\end{itemize}
Applying Jensen's inequality to the second term in the numerator, we obtain
\begin{align}
    \sum_{k \in [K]\backslash \singleton{y} } \frac{1}{B_k} \sum_{m \in \mathcal{B}_k}\exp(z_i \cdot z_m) \ge \sum_{k \in [K]\backslash \singleton{y} } \exp\left(\frac{1}{B_k} \sum_{m \in \mathcal{B}_k}z_i \cdot z_m\right),  \label{lem1-bound-2}
\end{align}
where equality is attained if and only if
\begin{itemize}
    \item[(b)] for each \(k \in [K] \backslash \singleton{y}\), there is a constant \(\delta'_{k,i}\) such that \(z_i \cdot z_m = \delta'_{k,i}, \text{ for all } m \in \mathcal{B}_k\).
\end{itemize}
Applying the Cauchy-Schwarz inequality to the denominator, we obtain
\begin{align}
    \frac{1}{\exp(\frac{1}{B_y -1}\sum_{j \in \mathcal{B}_y \backslash \singleton{i}}z_i \cdot z_j)} \ge \frac{1}{\exp(\frac{1}{B_y -1}\sum_{j \in \mathcal{B}_y \backslash \singleton{i}} \|z_i\| \|z_j\|)} = \frac{1}{\exp(1)} = e^{-1},  \label{lem1-bound-3}
\end{align}
where equality is attained if and only if
\begin{itemize}
    \item[(c)] \(z_i = z_j\) for all \(i, j \in \mathcal{B}_y\).
\end{itemize}
Substituting the three inequalities \eqref{lem1-bound-1}, \eqref{lem1-bound-2}, and \eqref{lem1-bound-3} into the rewritten anchor-wise loss yields
\begin{align}
    \log &\left(\frac{\sum_{p \in \mathcal{B}_y\backslash\singleton{i}}\frac{w_y}{B_y-1}\exp(z_i \cdot z_p) + \sum_{k \in [K]\backslash\singleton{y}} \sum_{m \in \mathcal{B}_k}\frac{w_k}{B_k}\exp(z_i \cdot z_m)}{\exp\left(\frac{1}{B_y-1}\sum_{j \in \mathcal{B}_y \backslash\singleton{i}}z_i \cdot z_j\right)}\right)\\
    &\overset{\text{Eq.} \eqref{lem1-bound-1}}{\ge}  \log \left(\frac{w_y \exp\left(\frac{1}{B_y-1} \sum_{p \in \mathcal{B}_y\backslash\singleton{i}} z_i \cdot z_p\right) + \sum_{k \in [K]\backslash\singleton{y}} \sum_{m \in \mathcal{B}_k}\frac{w_k}{B_k}\exp(z_i \cdot z_m)}{\exp\left(\frac{1}{B_y-1}\sum_{j \in \mathcal{B}_y \backslash\singleton{i}}z_i \cdot z_j\right)}\right)\\
    &= \log \left(w_y + \frac{\sum_{k \in [K]\backslash\singleton{y}} \sum_{m \in \mathcal{B}_k}\frac{w_k}{B_k}\exp(z_i \cdot z_m)}{\exp\left(\frac{1}{B_y-1}\sum_{j \in \mathcal{B}_y \backslash\singleton{i}}z_i \cdot z_j\right)}\right) \\
    &\overset{\text{Eq.} \eqref{lem1-bound-2}}{\ge}  \log \left(w_y + \frac{ \sum_{k \in [K]\backslash\singleton{y}} w_k \exp\left(\frac{1}{B_k}\sum_{m \in \mathcal{B}_k} z_i \cdot z_m\right)}{\exp\left(\frac{1}{B_y-1}\sum_{j \in \mathcal{B}_y \backslash\singleton{i}}z_i \cdot z_j\right)}\right)\\
    &\overset{\text{Eq.} \eqref{lem1-bound-3}}{\ge}  \log \left(w_y + e^{-1}\sum_{k \in [K]\backslash\singleton{y}} w_k \exp\left(z_i \cdot \bar{z}_k\right)\right),
\end{align}
where equality is attained if and only if conditions~\ref{(C2)} and~\ref{(C3)} are satisfied.
\end{proof}
\subsubsection{Lower bound of the class-wise loss (Lemma~\ref{appendix-lem-class-wise})}
Lemma~\ref{appendix-lem-class-wise} provides a lower bound for the class-wise loss obtained by averaging the anchor-wise lower bound over a single class.
\begin{applem}\label{appendix-lem-class-wise}(Lower bound of the class-wise loss)
    Let \(Z, \bar{Z}, Y\) be defined as in Proposition~\ref{appendix-prop-scl}. For any \(Z\) and \(Y\), fix a class \(y\) and let \(w_k > 0\) for all \(k \in [K]\). Then,
    \begin{align}
        \frac{1}{B_y}\sum_{i \in \mathcal{B}_y} \log\!\left(\! w_y + e^{-1}\!\!\!\! \sum_{k \in [K]\backslash \singleton{y}}\!\!\!w_k \exp(z_i \cdot \bar{z}_k) \!\!\right) \ge\! \log\!\left( \!w_y + e^{-1}\!\!\!\!\sum_{k \in [K]\backslash \singleton{y}}w_k\exp(\bar{z}_y \cdot \bar{z}_k)\! \right),
    \end{align}
    where equality is attained if and only if the following condition is satisfied:
    \begin{enumerate}[label=(C\arabic*), resume*=LBconditions]
        \item For each \(k \in [K]\backslash\singleton{y}\), there is a constant \(\delta_k\) such that \(z_i \cdot \bar{z}_k = \delta_k\), for all \(i\in \mathcal{B}_y\). \label{(C4)}
    \end{enumerate}
\end{applem}
\begin{proof}
The inequality follows from applying Jensen's inequality to the function 
\[f(z) = \log \left(w_y + e^{-1} \sum_{k \in [K] \backslash \singleton{y}}w_k \exp(z \cdot \bar{z}_k)\right).\]
Since \(f\) is convex, we have \(\frac{1}{B_y}\sum_{i \in \mathcal{B}_y}f(z_i) \ge f\left( \frac{1}{B_y}\sum_{i\in \mathcal{B}_y}z_i\right)\), which yields the desired result. Since all weights are positive, $f$ is strictly convex in the logits $(z \cdot \bar{z}_k)_{k \neq y}$, and thus equality holds if and only if \ref{(C4)} holds.
\end{proof}
\subsubsection{Proofs of Propositions~\ref{appendix-prop-scl},~\ref{appendix-prop-bcl}, and~\ref{appendix-prop-FBCL}}
We now derive the class-level lower bounds of the three losses by applying Lemma~\ref{appendix-lem-anchor-wise} to bound each anchor-wise loss, and then Lemma~\ref{appendix-lem-class-wise} to bound each class-wise loss. The three propositions differ in their negative class weights and the aggregation of anchor-wise losses. Combining the equality conditions~\ref{(C2)},~\ref{(C3)}, and~\ref{(C4)} across all \(y \in [K]\) yields the class-collapse condition~\ref{(C1)-appendix}.
\begin{enumerate}[label=(C\arabic*)]
        \item (\textbf{Class-collapse condition}) For all \(y \in [K]\), \(z_i = \bar{z}_y\) for all \(i \in \mathcal{B}_y\). \label{(C1)-appendix}
\end{enumerate}
\begin{appprop}\label{appendix-prop-scl}(Lower bound of SCL loss)
    Let \(Z=(z_1, \dots, z_B)\in(\mathbb{S}^{h-1})^B\) be a \(B\)-point configuration with labels \(Y=(y_1, \dots, y_B) \in [K]^B\). Then,
    \begin{align}\label{appendix-lowerbound-scl}
        \mathcal{L}_{\mathrm{SCL}}(Z;Y, \mathcal{B}) \ge \frac{1}{B} \sum_{y \in [K]} B_y \log \left( B_y-1 + e^{-1}\sum\limits_{k \in [K]\backslash \singleton{y}} B_k \exp(\bar{z}_y \cdot \bar{z}_k) \right) \eqcolon \mathcal{LB}_{\mathrm{SCL}}(\bar{Z};K),
    \end{align}
    where \(\bar{z}_y = \frac{1}{B_y}\sum_{i \in \mathcal{B}_y}z_i\) and \(\bar{Z}=(\bar{z}_1, \dots, \bar{z}_K)\). Equality is attained if and only if condition~\ref{(C1)-appendix} is satisfied.
\end{appprop}
\paragraph{Proof of Proposition~\ref{appendix-prop-scl}.}
\begin{proof}
Setting \(w_k = B_k\) (with \(w_y = B_y - 1\) when \(k = y\)), we apply Lemmas~\ref{appendix-lem-anchor-wise} and~\ref{appendix-lem-class-wise} sequentially:
\begin{align}
    \mathcal{L}_{\mathrm{SCL}}(Z;Y, \mathcal{B}) &= -\frac{1}{B} \sum_{y \in [K]} \sum\limits_{i \in \mathcal{B}_y} \frac{1}{B_y - 1} \sum_{j \in \mathcal{B}_y \backslash \singleton{i}}\log \left( \frac{\exp(z_i \cdot z_j)}{\sum_{k \in \mathcal{B} \backslash \singleton{i}}\exp(z_i \cdot z_k)} \right)\\
    &\overset{\text{Lem.~}\ref{appendix-lem-anchor-wise}}{\ge} \frac{1}{B} \sum_{y \in [K]} \sum\limits_{i \in \mathcal{B}_y}\log \left( B_y-1 + e^{-1}\sum_{k \in [K]\backslash\singleton{y}}B_k \exp(z_i \cdot \bar{z}_k)\right)\\
    &= \frac{1}{B} \sum_{y \in [K]} B_y \frac{1}{B_y} \sum\limits_{i \in \mathcal{B}_y}\log \left( B_y-1 + e^{-1}\sum_{k \in [K]\backslash\singleton{y}}B_k \exp(z_i \cdot \bar{z}_k)\right) \\
    &\overset{\text{Lem.~}\ref{appendix-lem-class-wise}}{\ge} \frac{1}{B} \sum_{y \in [K]} B_y \log \left( B_y-1 + e^{-1}\sum_{k \in [K]\backslash\singleton{y}}B_k \exp(\bar{z}_y \cdot \bar{z}_k)\right).
\end{align}
\end{proof}
\begin{appprop}\label{appendix-prop-bcl}(Lower bound of BCL loss)
    Let \(Z, \bar{Z}, Y\) be defined as in Proposition~\ref{appendix-prop-scl}. Then,
    \begin{align}\label{appendix-lowerbound-bcl}
        \mathcal{L}_{\mathrm{BCL}}(Z;Y, \mathcal{B}) \ge \frac{1}{B} \sum_{y \in [K]} B_y \log \left( 1 + e^{-1}\sum\limits_{k \in [K]\backslash \singleton{y}}  \exp(\bar{z}_y \cdot \bar{z}_k) \right) \eqcolon \mathcal{LB}_{\mathrm{BCL}}(\bar{Z};K).
    \end{align}
    Equality is attained if and only if condition~\ref{(C1)-appendix} is satisfied.
\end{appprop}
\paragraph{Proof of Proposition~\ref{appendix-prop-bcl}.}
\begin{proof}
Setting \(w_k = 1\) for all \(k \in [K]\), we apply Lemmas~\ref{appendix-lem-anchor-wise} and~\ref{appendix-lem-class-wise} sequentially:
\begin{align}
    \mathcal{L}_{\mathrm{BCL}}(Z;Y, \mathcal{B}) &= -\frac{1}{B} \sum_{y \in [K]} \sum\limits_{i \in \mathcal{B}_y} \frac{1}{B_y - 1} \sum_{j \in \mathcal{B}_y \backslash \singleton{i}}\log \left( \frac{\exp(z_i \cdot z_j)}{\sum_{k \in [K]} \frac{1}{B_k} \sum_{m \in \mathcal{B}_k}\exp(z_i \cdot z_m)} \right)\\
    &\overset{\text{Lem.~}\ref{appendix-lem-anchor-wise}}{\ge} \frac{1}{B} \sum_{y \in [K]} \sum\limits_{i \in \mathcal{B}_y}\log \left( 1 + e^{-1}\sum_{k \in [K]\backslash\singleton{y}} \exp(z_i \cdot \bar{z}_k)\right)\\
    &= \frac{1}{B} \sum_{y \in [K]} B_y \frac{1}{B_y} \sum\limits_{i \in \mathcal{B}_y}\log \left(1 + e^{-1}\sum_{k \in [K]\backslash\singleton{y}} \exp(z_i \cdot \bar{z}_k)\right) \\
    &\overset{\text{Lem.~}\ref{appendix-lem-class-wise}}{\ge} \frac{1}{B} \sum_{y \in [K]} B_y \log \left(1 + e^{-1}\sum_{k \in [K]\backslash\singleton{y}} \exp(\bar{z}_y \cdot \bar{z}_k)\right).
\end{align}
\end{proof}
\begin{appprop}\label{appendix-prop-FBCL}(Lower bound of FBCL loss)
    Let \(Z, \bar{Z}, Y\) be defined as in Proposition~\ref{appendix-prop-scl}. Then,
    \begin{align}
        \mathcal{L}_{\mathrm{FBCL}}(Z;Y, \mathcal{B}) \ge \frac{1}{K} \sum_{y \in [K]}\log \left( 1 + e^{-1}\sum\limits_{k \in [K]\backslash \singleton{y}}  \exp(\bar{z}_y \cdot \bar{z}_k) \right) \eqcolon \mathcal{LB}_{\mathrm{FBCL}}(\bar{Z};K).
    \end{align}
    Equality is attained if and only if condition~\ref{(C1)-appendix} is satisfied.
\end{appprop}
\paragraph{Proof of Proposition~\ref{appendix-prop-FBCL}.}
\begin{proof}
Setting \(w_k = 1\) for all \(k \in [K]\), we apply Lemmas~\ref{appendix-lem-anchor-wise} and~\ref{appendix-lem-class-wise} sequentially:
\begin{align}
    \mathcal{L}_{\mathrm{FBCL}}(Z;Y, \mathcal{B}) &= -\frac{1}{K} \sum_{y \in [K]} \frac{1}{B_y}\sum\limits_{i \in \mathcal{B}_y} \frac{1}{B_y - 1} \sum_{j \in \mathcal{B}_y \backslash \singleton{i}}\log \left( \frac{\exp(z_i \cdot z_j)}{\sum_{k \in [K]} \frac{1}{B_k} \sum_{m \in \mathcal{B}_k}\exp(z_i \cdot z_m)} \right)\\
    &\overset{\text{Lem.~}\ref{appendix-lem-anchor-wise}}{\ge} \frac{1}{K} \sum_{y \in [K]} \frac{1}{B_y}\sum\limits_{i \in \mathcal{B}_y}\log \left( 1 + e^{-1}\sum_{k \in [K]\backslash\singleton{y}} \exp(z_i \cdot \bar{z}_k)\right)\\
    &\overset{\text{Lem.~}\ref{appendix-lem-class-wise}}{\ge} \frac{1}{K} \sum_{y \in [K]} \log \left(1 + e^{-1}\sum_{k \in [K]\backslash\singleton{y}} \exp(\bar{z}_y \cdot \bar{z}_k)\right).
\end{align}
\end{proof}
\subsection{Decomposition of losses}\label{appendix-decomposition}
We rewrite each loss by separating the numerator and the denominator inside the \(\log\) into an alignment term \(\mathcal{L}^{\text{alignment}}\) and a uniformity term \(\mathcal{L}^{\text{uniformity}}\) (we discard \((Z;Y, \mathcal{B})\) for brevity):
\begin{align}
    \mathcal{L}_{\mathrm{SCL}} &= \mathcal{L}_{\mathrm{SCL}}^{\text{alignment}} + \mathcal{L}_{\mathrm{SCL}}^{\text{uniformity}}, \label{eq:decomp-scl}\\
    \mathcal{L}_{\mathrm{BCL}} &= \mathcal{L}_{\mathrm{BCL}}^{\text{alignment}} + \mathcal{L}_{\mathrm{BCL}}^{\text{uniformity}}, \label{eq:decomp-bcl}\\
    \mathcal{L}_{\mathrm{FBCL}} &= \mathcal{L}_{\mathrm{FBCL}}^{\text{alignment}} + \mathcal{L}_{\mathrm{FBCL}}^{\text{uniformity}}, \label{eq:decomp-fbcl-appendix}
\end{align}
where the alignment terms are given by
{\small
\begin{align}
    \mathcal{L}_{\mathrm{SCL}}^{\text{alignment}}
    = \mathcal{L}_{\mathrm{BCL}}^{\text{alignment}}
    &= -\frac{1}{B}\sum_{y \in [K]} \sum_{i\in \mathcal{B}_y}
    \frac{1}{B_y-1}\sum_{j\in \mathcal{B}_y\backslash \singleton{i}} (z_i \cdot z_j),
    \\
    \mathcal{L}_{\mathrm{FBCL}}^{\text{alignment}}
    &= -\frac{1}{K}\sum_{y \in [K]} \frac{1}{B_y}\sum_{i\in \mathcal{B}_y}
    \frac{1}{B_y-1}\sum_{j\in \mathcal{B}_y\backslash \singleton{i}} (z_i \cdot z_j),
\end{align}
}
and the uniformity terms are given by
{\small
\begin{align}
    \mathcal{L}_{\mathrm{SCL}}^{\text{uniformity}}
    &= \frac{1}{B}\sum_{y \in [K]} \sum_{i\in \mathcal{B}_y}
    \log\left(\sum_{k\in \mathcal{B}\backslash \singleton{i}}\exp(z_i\cdot z_k)\right), \\
    \mathcal{L}_{\mathrm{BCL}}^{\text{uniformity}}
    &= \frac{1}{B}\sum_{y \in [K]} \sum_{i\in \mathcal{B}_y}
    \log\left(\sum_{k\in [K]}\frac{1}{B_k}\sum_{m\in \mathcal{B}_k}\exp(z_i\cdot z_m)\right), \\
    \mathcal{L}_{\mathrm{FBCL}}^{\text{uniformity}}
    &= \frac{1}{K}\sum_{y \in [K]} \frac{1}{B_y}\sum_{i\in \mathcal{B}_y}
    \log\left(\sum_{k\in [K]}\frac{1}{B_k}\sum_{m\in \mathcal{B}_k}\exp(z_i\cdot z_m)\right).
\end{align}
}
The uniformity term cannot be feasibly written with respect to the intra-class variance without any assumptions. We hope that future work could assume a distribution of the features and further analyze how the uniformity term penalizes intra-class variances.
\paragraph{Alignment term as a function of intra-class variance.}
Let \(V_y = \frac{1}{B_y}\sum_{i \in \mathcal{B}_y} \|z_i - \bar{z}_y \|^2\). The alignment term of SCL or BCL (we only write for the SCL case) can be written using the following identities:
\begin{align}
    V_y &= \frac{1}{B_y} \sum_{i\in \mathcal{B}_y}\|z_i - \bar{z}_y\|^2 \\
    &= \frac{1}{B_y } \sum_{i \in \mathcal{B}_y} \left(\|z_i \|^2 - 2 z_i \cdot \bar{z}_y + \|\bar{z}_y\|^2\right)\\
    &= \frac{1}{B_y } \sum_{i \in \mathcal{B}_y} 1 - \frac{2}{B_y }\sum_{i \in \mathcal{B}_y} z_i \cdot \bar{z}_y + \frac{1}{B_y }\sum_{i \in\mathcal{B}_y} \|\bar{z}_y\|^2\\
    &=1 - \|\bar{z}_y \|^2,\\[10pt]
    \sum_{i\in \mathcal{B}_y}\sum_{j \in \mathcal{B}_y \backslash \singleton{i}}z_i \cdot z_j &= \sum_{i \in \mathcal{B}_y} \sum_{j \in \mathcal{B}_y} z_i \cdot z_j - \sum_{i \in \mathcal{B}_y} \|z_i \|^2 \\
    &= \left\|\sum_{i \in \mathcal{B}_y}z_i \right\|^2 - \sum_{i \in \mathcal{B}_y}\|z_i\|^2\\[5pt]
    &= B_y^2\|\bar{z}_y\|^2 - B_y \\[5pt]
    &= B_y^2 (\|\bar{z}_y \|^2 - 1) + B_y^2 - B_y \\[5pt]
    &= -B_y^2 (V_y) + B_y(B_y-1).
\end{align}
Using the above,
\begin{align}
    \mathcal{L}_{\mathrm{SCL}}^{\text{alignment}} &= -\frac{1}{B}\sum_{y \in [K]} \sum_{i\in \mathcal{B}_y}
    \frac{1}{B_y-1}\sum_{j\in \mathcal{B}_y\backslash \singleton{i}} z_i\cdot z_j\\
    &= -\frac{1}{B}\sum_{y \in [K]} 
    \frac{1}{B_y-1} \sum_{i\in \mathcal{B}_y}\sum_{j \in \mathcal{B}_y \backslash \singleton{i}}z_i \cdot z_j \\
    &= - \frac{1}{B}\sum_{y \in [K]} \frac{1}{B_y-1}
    \big(-B_y^2 (V_y) + B_y(B_y-1)\big)\\
    &= \sum_{y \in [K]} \frac{B_y^2}{B(B_y-1)}V_y -
     \sum_{y \in [K]}\frac{1}{B}\frac{B_y(B_y-1)}{(B_y-1)}\\
    &= \sum_{y \in [K]} \frac{B_y^2}{B(B_y-1)}V_y - 1.
\end{align}
Similarly, 
\begin{align}
    \mathcal{L}_{\mathrm{FBCL}}^{\text{alignment}} &= -\frac{1}{K}\sum_{y \in [K]} \frac{1}{B_y} \sum_{i\in \mathcal{B}_y}
    \frac{1}{B_y-1}\sum_{j\in \mathcal{B}_y\backslash \singleton{i}} z_i\cdot z_j\\
    &= -\frac{1}{K}\sum_{y \in [K]} 
    \frac{1}{B_y(B_y-1)} \sum_{i\in \mathcal{B}_y}\sum_{j \in \mathcal{B}_y \backslash \singleton{i}}z_i \cdot z_j \\
    &= - \frac{1}{K}\sum_{y \in [K]} \frac{1}{B_y(B_y-1)}
    \big(-B_y^2 (V_y) + B_y(B_y-1)\big)\\
    &= \sum_{y \in [K]} \frac{B_y}{K(B_y-1)}V_y -
     \sum_{y \in [K]}\frac{1}{K}\frac{B_y(B_y-1)}{B_y(B_y-1)}\\
    &= \sum_{y \in [K]} \frac{B_y}{K(B_y-1)}V_y - 1.
\end{align}
\subsection{Balancing components in FBCL}\label{Appendix-balancing_components}
As described in Section~\ref{method-sb}, FBCL contains three balancing components: positive balancing (PB), which averages over positive samples as in SCL; negative balancing (NB), which averages negative samples per class in the denominator as in BCL; and structural balancing (SB), which normalizes the aggregation of anchor-wise losses. The following equation marks each method within the FBCL loss:
\begin{align}
    \mathcal{L}_{\mathrm{FBCL}}(Z;Y, \mathcal{B}) = -\frac{1}{K} \sum_{y \in [K]} \;
    \underbrace{\frac{1}{B_y}}_{\text{SB}}\;
    \sum_{i \in \mathcal{B}_y} \;
    \underbrace{\frac{1}{B_y - 1}\sum_{j \in \mathcal{B}_y \backslash \{i\}}}_{\text{PB}}\;
    \log \Big( \frac{\exp(z_i \cdot z_j)}{\sum\limits_{ k \in [K] }\;
    \underbrace{\frac{1}{B_k}}_{\text{NB}}\; \sum\limits_{m \in \mathcal{B}_k}
    \exp(z_i \cdot z_m)} \Big).\notag
\end{align}
For positive balancing, we include the summation of positive samples so that a method without balancing corresponds to SimCLR.
\section{Experimental Details}
\subsection{Figure~\ref{fig:regular-simplex} Details}\label{Appendix-Fig1}
We use three random classes from ImageNet, and resize the images to $32\times32$ and train a 2-layer MLP that maps each flattened image through a hidden layer of dimension 256 to a two-dimensional normalized embedding. Thus, the learned features lie on the unit circle in the two-dimensional feature space. We therefore visualize the embeddings directly on this circle. We train all models using the SGD optimizer with learning rate 0.1, weight decay \(5e\!-\!4\), momentum 0.9.
\subsection{Implementation Details and Hyperparameters}\label{Appendix-implementation}
\paragraph{FBCL loss in practice.}
As we discussed in Section~\ref{sec:discussion}, our analysis follows the standard setting where each batch contains all classes. We also note that the temperature rate is fixed to $\tau=1$ in our theory, since the value does not affect the theory. In our settings, we adopt the temperature rate used in the baseline methods for a fair comparison. But in practice, the representation space is usually affected by the value of the temperature rate under both balanced and imbalanced distributions \citep{wang2021understanding,kukleva2023temperature}. We hope that further studies investigate the attainability of a regular simplex under a large-batch training regime and its relation to temperature. Since the goal of our work is not to attain state-of-the-art results, we carefully conduct experiments to do a fair comparison with the baseline methods. We use the following loss in practice:
{\small
\begin{align}
    \mathcal{L}_{\mathrm{FBCL}}(Z;Y,\mathcal{B}) =- \frac{1}{|\mathcal{Y}_\mathcal{B}|} \sum_{y \in \mathcal{Y}_\mathcal{B}} \frac{1}{B_y}\sum_{i \in \mathcal{B}_y}
    \frac{1}{B_y} \sum_{j \in (\mathcal{B}_y\cup \singleton{c_y}) \backslash \{i\}}\!\!\!\!\!\!\!\!\log \Bigg( \frac{\exp(z_i \!\cdot\! z_j/\tau)}{\sum\limits_{ k \in [K] }\! \frac{1}{B_k+1}\! \!\!\!\!\!\sum\limits_{m \in (\mathcal{B}_k\cup \singleton{c_k})}\!\!\!\!\!\exp(z_i \!\cdot\! z_m/\tau)} \!\Bigg), 
\end{align}}
where $\mathcal{Y}_\mathcal{B}$ is the index set of class labels in the batch $\mathcal{B}$, and $\{c_y\}$ are class-complemented pairs. The difference between practice and our theory lies in the class distribution of the anchors. Our theory assumes that the anchor classes cover the entire label set, \([K]\), whereas in practice, the set of anchor classes depends on the mini-batch composition. \textbf{Thus, increasing the batch size is expected to improve class coverage and make the practical objective potentially encouraging a regular simplex geometry even on large-scale datasets}. In our experiments, however, the batch size was limited by GPU memory, and we leave a systematic investigation of this regime to future work. To further analyze this regime, simply increasing the batch size may still be insufficient for severe long-tailed datasets. For example, ImageNet-LT contains 1,000 classes, and even with a batch size of 1,024, few-shot classes may appear only rarely as anchors. One possible extension is therefore to augment the anchor set with class prototypes of few-shot classes. In particular, the same class prototypes introduced as complementary contrastive pairs could serve as anchors for few-shot classes, increasing their anchor coverage without requiring substantially larger mini-batches. We leave the theoretical and empirical study of such prototype-based anchor augmentation for future work.
\paragraph{ImageNet-LT and iNaturalist.}
Our setting mainly follows \citet{zhu2022balanced}. When combining with other methods, we follow their settings and add structural balancing. For ImageNet-LT, we train the model for 90 epochs with learning rate 0.1, and weight decay \(5 \!\times\! 10^{-4}\). For iNaturalist, we train for 100 epochs with learning rate 0.2, and weight decay \(1 \!\times\! 10^{-4}\). For both datasets, we use a batch size of 256, cosine scheduling for learning rate, and set \(\lambda=1.0\) and \(\mu=0.35\). Class-complement \citep{zhu2022balanced} is used during training to incorporate class centers as additional contrastive pairs, allowing all classes to appear in every mini-batch. For SCL, class-complement is not used as \citet{zhu2022balanced} reported performance degradation when used. SimCLR also does not use class-complement as it is a self-supervised method. The temperature is set to \(\tau=0.07\). All models are trained using SGD with momentum 0.9, using four RTX 4090 GPUs.
\paragraph{CIFAR-100-LT.}
For CIFAR-100, we train the model for 200 epochs with learning rate warmed up to 0.15 within the first 5 epochs, decaying at 160 and 180 with a step size 0.1, and weight decay \(5 \times 10^{-4}\). We set the temperature to \(\tau=0.1\), and set \(\lambda=2.0\) and \(\mu=0.6\). All models are trained using SGD with momentum 0.9, using one RTX 4090 GPU.
\subsection{Transfer Learning Evaluation protocol}\label{Appendix-eval-protocol}
\paragraph{Linear evaluation}
The linear evaluation protocol follows \citet{chen2020simple, oh2024effectiveness}. Specifically, we divide the training dataset into a train set and a validation set as shown in Table~\ref{table:dataset-info} and tune the regularization parameter by minimizing the L2-regularized cross-entropy objective using L-BFGS. Images are only resized and \(224\times224\) center-cropped. With the best parameter, we retrain the linear classifier using both training and validation splits and then report the test accuracy. 
\paragraph{Fine-tuning protocol}
We divide the training dataset into a train set and a validation set as in linear evaluation. Training images are resized and randomly cropped to \(224 \times 224\), and then randomly horizontally flipped. We perform grid search over the number of epochs, learning rate, and weight decay using the validation split with SGD with momentum \(0.9\). With the best hyperparameters, we retrain using both training and validation splits and then report the model accuracy on the test set. 
\begin{table}[h]
\caption{Dataset information}
\label{table:dataset-info}
\centering
\begin{tabular}{@{}l|ccccc@{}}
\toprule
Dataset & \# of classes & Training & Validation & Test & Metric \\ \midrule
CIFAR-10 & 10 & 45000 & 5000 & 10000 & Top-1 accuracy \\
CIFAR-100 & 100 & 45000 & 5000 & 10000 & Top-1 accuracy \\
MIT67 & 67 & 4690 & 670 & 1340 & Top-1 accuracy \\
DTD (split1) & 47 & 1880 & 1880 & 1880 & Top-1 accuracy \\
CUB200 & 200 & 4994 & 1000 & 5794 & Top-1 accuracy \\
Food & 101 & 68175 & 7575 & 25250 & Top-1 accuracy \\
Cars & 196 & 6494 & 1650 & 8041 & Top-1 accuracy \\
Pets & 37 & 2940 & 740 & 3669 & Mean per-class accuracy \\
Caltech101 & 101 & 2525 & 505 & 5647 & Mean per-class accuracy \\
Flowers & 102 & 1020 & 1020 & 6149 & Mean per-class accuracy \\
Aircraft & 100 & 3334 & 3333 & 3333 & Mean per-class accuracy \\ \bottomrule
\end{tabular}
\end{table}
\subsection{Toy Experiment}\label{Appendix-toy-exp}
The detailed settings of the toy dataset used in Section~\ref{method-icv} follow \citet{oh2024effectiveness}. The dataset consists of five classes, each following a Gaussian distribution with orthogonal mean vectors and a shared isotropic covariance matrix with a scale of 0.15. The mean vectors are obtained from the left singular vectors of a random matrix sampled from a standard Gaussian. The data dimension is set to 6, and the class cardinalities are \(500,250,100,50,25\). The first class is treated as the many class, the second class as the medium class, and the remaining three classes as few classes. The test dataset consists of 500 samples, with 100 samples per class.

We use a two-layer MLP backbone with 32 hidden dimensions, followed by a classifier. For the cross-entropy baseline, we train the backbone and classifier for 2000 epochs. For contrastive objectives, we first train the backbone model for 2000 epochs and then train the classifier for 400 epochs. We use SGD with learning rate \(0.1\), weight decay \(5e-4\), and temperature \(0.25\) for contrastive training.

For the intra-class variance analysis, we first compute the within-class variance and the between-class variance from the learned representations. Let \(S_W\) and \(S_B\) denote the within-class and between-class variances, respectively, and let \(S_T = S_W + S_B\) be the total variance. For each class \(y\), we compute the class-wise intra-class variance and normalize it using \(S_T \). We then report the average normalized intra-class variance for each group. This normalization allows us to compare how much variance each class retains relative to the overall representation variance.

\subsection{More details and standard deviations}\label{Appendix-statistics}
\subsubsection{FBCL on ImageNet-LT and iNaturalist}\label{Appendix-statistics-fbcl_imagenet_inat}
Here, statistics and iNaturalist results for Table~\ref{table:contrastive_comparison} are provided.
\begin{table}[h]
\centering
\caption{Statistics over three runs, on contrastive objectives on ImageNet-LT and iNaturalist.}
\resizebox{\linewidth}{!}{%
\begin{tabular}{@{}l|cccc|cccc@{}}
\toprule
\multirow{2}{*}{\textbf{Method}} & \multicolumn{4}{c|}{ImageNet-LT} & \multicolumn{4}{c}{iNaturalist} \\ \cmidrule(l){2-9} 
 & Many & Medium & Few & All & Many & Medium & Few & All \\ \midrule
SimCLR \citep{chen2020simple} & 65.0 {\small$\pm$0.1} & 52.6 {\small$\pm$0.2} & 36.2 {\small$\pm$0.1} & 55.2 {\small$\pm$0.2} & 68.1 {\small$\pm$0.3} & 70.7 {\small$\pm$0.2} & 69.9 {\small$\pm$0.3} & 70.1 {\small$\pm$0.2}\\
SCL \citep{khosla2020supervised} & 65.8 {\small$\pm$0.4} & 53.3 {\small$\pm$0.4}& 36.1 {\small$\pm$0.5} & 55.8 {\small$\pm$0.4} & 68.4 {\small$\pm$0.3} & 71.1 {\small$\pm$0.4}& 70.2 {\small$\pm$0.5}& 70.4 {\small$\pm$0.3}\\
BCL \citep{zhu2022balanced} & 65.8 {\small$\pm$0.2}& 53.7 {\small$\pm$0.2}& 36.2 {\small$\pm$0.1}& 56.0 {\small$\pm$0.1}& 70.0 {\small$\pm$0.7} & 72.0 {\small$\pm$0.4} & 71.1 {\small$\pm$0.2}& 71.3 {\small$\pm$0.6}\\
\rowcolor{gray!13} FBCL (Ours) & 66.4{\small$\pm$0.2} & 54.1 {\small$\pm$0.2}& 37.5 {\small$\pm$0.4}& 56.6 {\small$\pm$0.1}& 70.4 {\small$\pm$0.3}& 72.8 {\small$\pm$0.2}& 71.9 {\small$\pm$0.3}& 72.2 {\small$\pm$0.1}\\ \bottomrule
\end{tabular}}
\end{table}
\subsubsection{Intra-class variance and accuracy}\label{Appendix-statistics-icv}
Here, statistics for Table~\ref{table:icv}, including intra-class variance and accuracy on synthetic data, ImageNet-LT and CIFAR-100-LT (imbalance ratio=100), are provided in Tables~\ref{Appendix-statistics-table:icv_synthetic},~\ref{Appendix-statistics-table:icv_cifar}, and~\ref{Appendix-statistics-table:icv_imagenet}.
\begin{table}[h!]
\centering
\caption{Statistics over five runs on intra-class variance and Top-1 accuracy on synthetic data.}\label{Appendix-statistics-table:icv_synthetic}
\resizebox{0.7\linewidth}{!}{%
\begin{tabular}{l|ccc|c}
\toprule
\textbf{Method} & Many & Medium & Few & Acc \\
\midrule
CE & $0.459${\fontsize{7}{8}\selectfont\,$\pm$\,0.02} & $0.533${\fontsize{7}{8}\selectfont\,$\pm$\,0.03} & $0.721${\fontsize{7}{8}\selectfont\,$\pm$\,0.03} & $73.8${\fontsize{7}{8}\selectfont\,$\pm$\,0.16} \\ \midrule
SCL  & $0.272${\fontsize{7}{8}\selectfont\,$\pm$\,0.01} & $0.346${\fontsize{7}{8}\selectfont\,$\pm$\,0.02} & $0.575${\fontsize{7}{8}\selectfont\,$\pm$\,0.05} & $74.1${\fontsize{7}{8}\selectfont\,$\pm$\,2.74} \\
BCL  & $0.250${\fontsize{7}{8}\selectfont\,$\pm$\,0.02} & $0.385${\fontsize{7}{8}\selectfont\,$\pm$\,0.03} & $0.749${\fontsize{7}{8}\selectfont\,$\pm$\,0.07} & $74.1${\fontsize{7}{8}\selectfont\,$\pm$\,2.36} \\
\rowcolor{gray!13} FBCL (Ours) & $0.409${\fontsize{7}{8}\selectfont\,$\pm$\,0.02} & $0.453${\fontsize{7}{8}\selectfont\,$\pm$\,0.02} & $0.438${\fontsize{7}{8}\selectfont\,$\pm$\,0.03} & $75.5${\fontsize{7}{8}\selectfont\,$\pm$\,3.10} \\
\bottomrule
\end{tabular}}
\end{table}
\begin{table}[h!]
\centering
\caption{Statistics over five runs of the ablation study over balancing components on CIFAR-100-LT (IR=100).}\label{Appendix-statistics-table:icv_cifar}
\renewcommand{\arraystretch}{1.39}
\resizebox{0.7\linewidth}{!}{%
\begin{tabular}{lccc|ccc|c}
\toprule
\multirow{2}{*}{\textbf{Method}} & \multirow{2}{*}{\textbf{PB}} & \multirow{2}{*}{\textbf{NB}} & \multirow{2}{*}{\textbf{SB}} & \multicolumn{3}{c|}{\textbf{Intra-class variance}} & \multirow{2}{*}{\textbf{Acc.}} \\
\cmidrule(lr){5-7}
& & & & Many & Medium & Few & \\ \midrule
PB (SCL) & \cmark & \xmark & \xmark & 0.211$\pm$0.001 & 0.342$\pm$0.003 & 0.420$\pm$0.001 & 51.5$\pm$0.18 \\
PB+NB (BCL) & \cmark & \cmark & \xmark & 0.210$\pm$0.005 & 0.382$\pm$0.004 & 0.465$\pm$0.011 & 51.7$\pm$0.45 \\ \midrule
PB+SB & \cmark & \xmark & \cmark & 0.340$\pm$0.004 & 0.387$\pm$0.002 & 0.467$\pm$0.003 & 51.5$\pm$0.25 \\
NB+SB & \xmark & \cmark & \cmark & 0.660$\pm$0.009 & 0.620$\pm$0.006 & 0.601$\pm$0.009 & 51.4$\pm$0.16 \\
\rowcolor{gray!13} PB+NB+SB (Ours) & \cmark & \cmark & \cmark & 0.267$\pm$0.003 & 0.320$\pm$0.006 & 0.379$\pm$0.003 & 52.1$\pm$0.15 \\
\bottomrule
\end{tabular}}
\end{table}
\begin{table}[h!]
\centering
\caption{Ablation study of positive balancing (PB), negative balancing (NB), and structural balancing (SB) on ImageNet-LT. Statistics over three runs are given.}\label{Appendix-statistics-table:icv_imagenet}
\renewcommand{\arraystretch}{1.25}
\resizebox{\linewidth}{!}{%
\begin{tabular}{lccc|ccc|cccc}
\toprule
\multirow{2}{*}{\textbf{Method}}
& \multirow{2}{*}{\textbf{PB}}
& \multirow{2}{*}{\textbf{NB}}
& \multirow{2}{*}{\textbf{SB}}
& \multicolumn{3}{c|}{\textbf{Intra-class variance}}
& \multicolumn{4}{c}{\textbf{Acc.}} \\
\cmidrule(lr){5-7}
\cmidrule(lr){8-11}
& & & &
Many & Medium & Few &
Many & Medium & Few & All \\ \midrule
PB (SCL)
& \cmark & \xmark & \xmark
& 0.189$\pm$0.001
& 0.298$\pm$0.003
& 0.357$\pm$0.001
& 65.76$\pm$0.38
& 53.33$\pm$0.40
& 36.13$\pm$0.46
& 55.78$\pm$0.41 \\
PB+NB (BCL)
& \cmark & \cmark & \xmark
& 0.167$\pm$0.021
& 0.293$\pm$0.003
& 0.362$\pm$0.006
& 65.83$\pm$0.15
& 53.74$\pm$0.17
& 36.23$\pm$0.06
& 56.03$\pm$0.07 \\ \midrule
PB+SB
& \cmark & \xmark & \cmark
& 0.368$\pm$0.004
& 0.502$\pm$0.003
& 0.509$\pm$0.003
& 66.07$\pm$0.12
& 53.40$\pm$0.30
& 36.22$\pm$0.36
& 55.92$\pm$0.19 \\
NB+SB
& \xmark & \cmark & \cmark
& 0.597$\pm$0.002
& 0.623$\pm$0.002
& 0.602$\pm$0.001
& 65.20$\pm$0.20
& 53.03$\pm$0.29
& 35.47$\pm$0.32
& 55.35$\pm$0.17 \\
\rowcolor{gray!13}
PB+NB+SB (Ours)
& \cmark & \cmark & \cmark
& 0.196$\pm$0.008
& 0.281$\pm$0.010
& 0.339$\pm$0.008
& 66.44$\pm$0.15
& 54.12$\pm$0.21
& 37.50$\pm$0.40
& 56.64$\pm$0.12 \\
\bottomrule
\end{tabular}}
\end{table}
\subsubsection{Structural balancing with other methods}\label{Appendix-statistics-recent_methods}
Here, statistics for the structural balancing results in Table~\ref{table:main_results} are provided in Tables~\ref{Appendix-statistics-table:dbm} and~\ref{Appendix-statistics-table:proco}.

\begin{table}[h!]
\centering
\caption{Statistics over three runs of FBCL and BCL combined with DBM.}\label{Appendix-statistics-table:dbm}
\resizebox{1.0\linewidth}{!}{%
\begin{tabular}{@{}l|cccc|cccc@{}}
\toprule
\multirow{2}{*}{\textbf{Method}}  & \multicolumn{4}{c|}{\textbf{ImageNet-LT}} & \multicolumn{4}{c}{\textbf{iNaturalist}} \\ \cmidrule(l){2-9} 
& Many & Medium & Few & All & Many & Medium & Few & All \\ \midrule
DBM+BCL & 64.4 $\pm$ 0.1 & 54.3 $\pm$ 0.2 & 41.4 $\pm$ 0.3 & 56.5 $\pm$ 0.2 & 66.3 $\pm$ 0.3 & 71.8 $\pm$ 0.5 & 73.3 $\pm$ 0.3 & 71.8 $\pm$ 0.1
\\ \midrule
\rowcolor{gray!13}
DBM+FBCL (Ours) & 64.9 $\pm$ 0.2 & 54.5 $\pm$0.2 & 42.3 $\pm$ 0.3 & 56.9 $\pm$ 0.2 & 66.5 $\pm$ 0.2 & 72.4 $\pm$ 0.4 & 73.8 $\pm$ 0.2 & 72.3 $\pm$ 0.2
\\ \bottomrule
\end{tabular}}
\end{table}
\begin{table}[h!]
\centering
\caption{Statistics over three runs of structural balancing combined with ProCo on ImageNet-LT and iNaturalist.}\label{Appendix-statistics-table:proco}
\resizebox{1.0\linewidth}{!}{%
\begin{tabular}{@{}l|cccc|cccc@{}}
\toprule
\multirow{2}{*}{\textbf{Method}}     & \multicolumn{4}{c|}{\textbf{ImageNet-LT}}                     & \multicolumn{4}{c}{\textbf{iNaturalist}} \\ \cmidrule(l){2-9} 
 & Many & Medium & Few & All & Many & Medium & Few & All \\ \midrule
ProCo & 66.5$\pm$0.1 & 54.8$\pm$0.5 & 38.1$\pm$0.4 & 57.1$\pm$0.2 & 70.4$\pm$0.4 & 73.1$\pm$0.1 & 71.9$\pm$0.6 & 72.3$\pm$0.2 \\ \midrule
\rowcolor{gray!13} ProCo + SB (Ours) & 66.8$\pm$0.2 & 55.2$\pm$0.1 & 38.8$\pm$0.2 & 57.4$\pm$0.1 & 70.7$\pm$0.6 & 73.4$\pm$0.4 & 72.3$\pm$0.4 & 72.7$\pm$0.5 \\ \bottomrule
\end{tabular}}
\end{table}
We additionally report results on CIFAR-100-LT, with imbalance ratios in $\{100, 50, 10\}$, for \citet{son2025difficulty} and \citet{wang2024llm} in Table~\ref{Appendix-statistics-table:dbm_cifar} and for \citet{du2024probabilistic} in Table~\ref{Appendix-statistics-table:proco_cifar}.
\begin{table}[h!]
\centering
\resizebox{0.7\linewidth}{!}{%
\begin{tabular}{@{}l|ccc@{}}
\toprule
\multirow{2}{*}{\textbf{Method}}     & \multicolumn{3}{c}{\textbf{CIFAR-100-LT}}\\ \cmidrule{2-4}
 & IR=100 & IR=50 & IR=10 \\ \midrule
BCL               & $51.7 \pm 0.5$ & $55.7 \pm 0.4$ & $63.7 \pm 0.1$ \\
FBCL              & $52.1 \pm 0.2$ \textcolor{blue}{(+0.4)} & $56.0 \pm 0.2$ \textcolor{blue}{(+0.3)} & $64.0 \pm 0.1$ \textcolor{blue}{(+0.3)} \\ \midrule
DBM + BCL         & $50.4 \pm 1.0$ & $55.6 \pm 0.1$ & $65.7 \pm 0.2$  \\
DBM + FBCL        & $51.2 \pm 0.5$ \textcolor{blue}{(+0.8)} & $55.9 \pm 0.3$ \textcolor{blue}{(+0.3)} & $65.8 \pm 0.1$ \textcolor{blue}{(+0.1)}  \\ \midrule
LLM-AutoDA + BCL  & $50.6 \pm 0.6$ & $55.5 \pm 0.2$ & $64.8 \pm 0.1$  \\
LLM-AutoDA + FBCL & $51.3 \pm 0.3$ \textcolor{blue}{(+0.7)} & $55.9 \pm 0.3$ \textcolor{blue}{(+0.4)}& $65.2 \pm 0.5$ \textcolor{blue}{(+0.4)}  \\ \bottomrule
\end{tabular}
}\caption{Statistics over five runs of FBCL or BCL with other methods on CIFAR-100-LT.}\label{Appendix-statistics-table:dbm_cifar}
\end{table}
In the official implementation of \citet{wang2024llm}, code for ImageNet-LT or iNaturalist were not given. We will add the results when the codes are available. Note that FBCL corresponds to BCL$+$SB.
\begin{table}[h!]
\centering
\resizebox{0.5\linewidth}{!}{%
\begin{tabular}{@{}lccc@{}}
\toprule
\multicolumn{1}{l|}{}                 & \multicolumn{3}{c}{CIFAR-100-LT}                                                                     \\ \midrule
\multicolumn{1}{l|}{Imbalance Factor} & 100                                   & 50                                    & 10                   \\ \midrule
\textit{Top-1 accuracy}               & \multicolumn{1}{l}{}                  & \multicolumn{1}{l}{}                  & \multicolumn{1}{l}{} \\ \midrule
\multicolumn{1}{l|}{ProCo}   & $51.9 \pm 0.2 $ & $56.2 \pm 0.2$ & $64.6\pm 0.4$     \\ 
\multicolumn{1}{l|}{ProCo + SB}       & $52.1 \pm 0.3 $ & $56.3 \pm 0.21$ & $64.7 \pm 0.1$  \\ \midrule
\textit{Few-class accuracy}           & \multicolumn{1}{l}{}                  & \multicolumn{1}{l}{}                  & \multicolumn{1}{l}{} \\ \midrule
\multicolumn{1}{l|}{ProCo}  & 32.1 $\pm$ 0.6       & 43.2 $\pm$ 0.7   & -          \\
\multicolumn{1}{l|}{ProCo + SB} & 33.9 $\pm$ 0.6       & 44.8 $\pm$ 0.2    & -    \\ \bottomrule
\end{tabular}}\caption{Statistics over five runs of structural balancing applied to ProCo \citep{du2024probabilistic} on CIFAR-100-LT dataset. Adding structural balancing improves few class accuracy.}\label{Appendix-statistics-table:proco_cifar}
\end{table}
\subsection{Cosine similarity of the representation space}
\begin{table}[!h]\caption{Cosine similarity across class centers on CIFAR-100-LT.}\label{Appendix-statistics-table:geometry_diagnosis}
\centering
\begin{subtable}[t]{0.48\linewidth}
\centering
\caption{Average cosine similarity across imbalance ratios.}\label{Appendix-statistics-table:geometry_diagnosis-across_IR}
\renewcommand{\arraystretch}{0.97}
\resizebox{\linewidth}{!}{%
\begin{tabular}{@{}l|ccc@{}}
\toprule
\multirow{2}{*}{Method} & \multicolumn{3}{c}{CIFAR-100-LT} \\ \cmidrule(l){2-4} 
 & 100 ($\downarrow$) & 50 ($\downarrow$) & 10 ($\downarrow$) \\ \midrule
BCL & 0.298$\pm$0.003 & 0.296$\pm$0.003 & 0.294$\pm$0.007 \\
FBCL & \textbf{0.233}$\pm$0.004 & \textbf{0.240}$\pm$0.004 & \textbf{0.242}$\pm$0.006 \\ \bottomrule
\end{tabular}}
\end{subtable}
\begin{subtable}[t]{0.48\linewidth}
\centering
\caption{Average cosine similarity across classes.}\label{Appendix-statistics-table:geometry_diagnosis-across_classes}
\renewcommand{\arraystretch}{0.97}
\resizebox{\linewidth}{!}{%
\begin{tabular}{@{}l|ccc@{}}
\toprule
\multirow{2}{*}{Method} & \multicolumn{3}{c}{CIFAR-100-LT (IR=100)} \\ \cmidrule(l){2-4} 
 & Many ($\downarrow$) & Medium ($\downarrow$) & Few ($\downarrow$) \\ \midrule
BCL & 0.318$\pm$0.002 & 0.249$\pm$0.008 & 0.331$\pm$0.002 \\
FBCL & \textbf{0.259}$\pm$0.002 & \textbf{0.205}$\pm$0.004 & \textbf{0.230}$\pm$0.007 \\ \bottomrule
\end{tabular}}
\end{subtable}
\end{table}
\begin{table}[!h]
\centering
\caption{Cosine similarity across class centers on ImageNet-LT.}
\resizebox{0.65\linewidth}{!}{
\begin{tabular}{@{}l|cccc@{}}
\toprule
\multirow{2}{*}{\textbf{Method}} & \multicolumn{4}{c}{\textbf{ImageNet-LT}} \\ \cmidrule(l){2-5} 
 & Many ($\downarrow$) & Medium ($\downarrow$) & Few ($\downarrow$) & All ($\downarrow$) \\ \midrule
BCL & 0.154 {\small$\pm$0.003} & 0.106 {\small$\pm$0.004} & 0.104 {\small$\pm$0.005} & 0.124 {\small$\pm$0.004} \\
FBCL & \textbf{0.137} {\small$\pm$0.001} & \textbf{0.088} {\small$\pm$0.002} & \textbf{0.079} {\small$\pm$0.005} & \textbf{0.103} {\small$\pm$0.002} \\ \bottomrule
\end{tabular}}
\end{table}
As discussed in section~\ref{sec:discussion}, the attainability of a regular simplex in practice is beyond the scope of this work, so we instead evaluate the cosine similarities between class centers to analyze the representation space in which the contrastive objective operates, i.e., the output of the contrastive head. On CIFAR-100-LT, Table~\ref{Appendix-statistics-table:geometry_diagnosis} shows that FBCL consistently yields lower average cosine similarity across different imbalance ratios and classes. Lower cosine similarity indicates greater separation between class centers. Table~\ref{Appendix-statistics-table:geometry_diagnosis-across_classes} further shows the difference particularly in few-shot classes which implies that FBCL better separates tail classes than BCL. We also provide results on the large-scale dataset, ImageNet-LT. The results show the same trend, with FBCL consistently yielding lower cosine similarity across many-, medium-, and few-shot classes, as well as overall, indicating a more separated inter-class geometry than BCL.
\subsection{Class balanced sampling and other reweighting strategies}
\label{Appendix:Class-balanced_re-weighting}

One could expect structural balancing to be similar to conventional re-weighting, since both introduce class-dependent weighting, or that class-balanced sampling could provide a similar effect by constructing batches with a more uniform class distribution. However, prior studies have shown that re-balancing through class-balanced sampling or loss re-weighting (classification losses) can impair representation quality \citep{kang2019decoupling,zhou2020bbn}. This motivates examining whether similar strategies are effective for contrastive objectives and how structural balancing differs from them.

We therefore compare structural balancing with class-balanced sampling and representative re-weighting strategies. Class-balanced sampling (\textit{CB-sampling}) constructs batches with a more uniform class distribution, but changes the original data distribution and can degrade performance \citep{kang2019decoupling}. \textit{Re-weighting} applies a weight determined by the class prior to each loss $\frac{N}{KN_y}$, while class-balanced weighting (\textit{CB-weighting}) uses the effective number of samples, $\frac{1-\beta}{1-\beta^{N_y}}$ \citep{cui2019class}, with $\beta=0.99$. (SCL with CB-weighting corresponds to \citet{zhong2022rebalanced}). \textit{RCL} \citep{de2025long} instead introduces balancing based on class priors within the logarithmic term of the contrastive loss.
\begin{table}[h]
\centering
\caption{Comparison with class-balanced sampling and other re-weighting methods. Statistics over five runs are provided.}\label{Appendix:Class-balanced_re-weighting-table}
\resizebox{0.9\linewidth}{!}{
\begin{tabular}{@{}l|cccc@{}}
\toprule
\multirow{2}{*}{\textbf{Method}} & \multicolumn{4}{c}{\textbf{CIFAR-100-LT (IR=100)}} \\ \cmidrule(l){2-5} 
 & Many & Medium & Few & Overall \\ \midrule
BCL & 66.04$\pm$0.55 & 52.57$\pm$0.49 & 33.78$\pm$1.04 & 51.65$\pm$0.45 \\
BCL + CB Sampling & 51.32$\pm$0.25 & 56.63$\pm$0.38 & 30.17$\pm$1.60 & 46.84$\pm$0.39 \\
BCL + Re-weighting & 65.51$\pm$1.01 & 52.69$\pm$0.54 & 33.56$\pm$0.98 & 51.45$\pm$0.45 \\
SCL + CB-weighting & \multicolumn{1}{l}{64.84$\pm$0.80} & \multicolumn{1}{l}{51.44$\pm$0.52} & \multicolumn{1}{l}{33.14$\pm$0.64} & \multicolumn{1}{l}{50.65$\pm$0.27} \\
BCL + CB-weighting & 64.77$\pm$0.59 & 51.46$\pm$0.85 & 33.02$\pm$0.03 & 50.59$\pm$0.49 \\
RCL & 66.05$\pm$0.29 & 52.01$\pm$0.76 & 33.97$\pm$0.38 & 51.52$\pm$0.40 \\
\rowcolor{gray!13}
BCL + SB & 66.47$\pm$0.50 & 53.28$\pm$0.42 & 35.05$\pm$0.60 & 52.14$\pm$0.15 \\ \bottomrule
\end{tabular}}
\end{table}

As shown in Table~\ref{Appendix:Class-balanced_re-weighting-table}, class-balanced sampling yields the lowest overall accuracy, and interestingly, CB-weighting performs worse than conventional re-weighting, whereas structural balancing achieves the best performance. These results show that simply applying existing re-weighting strategies to a contrastive objective is not sufficient, as they do not provide the theoretical guarantee established for structural balancing. Note that in the full-batch setting, both are same, but such a setting is not practical for large-scale datasets, and in a mini-batch setting when re-weighting does not exactly compensate for the class imbalance, the resulting objective retains dependency on class frequency. In addition, \textbf{re-weighting does not consistently improve performance as the imbalance becomes more severe}. At IR=200, BCL$+$re-weighting achieves an overall accuracy of $46.35\pm0.34$, which is lower than both BCL ($47.63\pm0.30$) and FBCL ($48.18\pm0.32$). This suggests that conventional re-weighting is not always beneficial for contrastive objectives and can even degrade performance depending on class imbalance. Although structural balancing takes a similar form, it is derived from the class frequency dependence arising from the aggregation of contrastive losses. Our analysis shows that it removes this dependence in the lower bound and recovers the optimal representation geometry. Also, while re-weighting has been reported to impair representation learning when applied to classification losses, our results show that \textit{a re-weighting mechanism derived from the structure of a contrastive objective with a theoretical guarantee can instead be beneficial}, motivating further re-weighting strategies specifically designed for contrastive learning.
\end{document}